\documentclass[10pt,letterpaper]{paper}

\usepackage{lmodern}
\usepackage{silence}
\usepackage[authoryear, round]{natbib}
\usepackage[utf8]{inputenc}
\usepackage[T1]{fontenc}
\usepackage{microtype}

\makeatletter
\@namedef{ver@everyshi.sty}{}
\makeatother

\definecolor{custom_green}{rgb}{0.0, 0.5, 0.0}
\definecolor{custom_red}{rgb}{1.0, 0.01, 0.24}
\definecolor{title_blue}{HTML}{204899}   
\definecolor{cite_blue}{HTML}{044dc1}
\definecolor{cite_purple}{HTML}{7406a7}

\makeatletter
\renewcommand\footnoterule{%
  \kern 15\p@
  \hrule \@width 2in \kern 2.6\p@
  \vspace{4pt}
}
\makeatother

\usepackage{listings}
\usepackage[most]{tcolorbox}

\tcbset{
  promptstyle/.style={colback=gray!4, colframe=gray!55, coltitle=black, colbacktitle=gray!15},
  hadstyle/.style={colback=green!4, colframe=green!45!black, coltitle=black, colbacktitle=green!12},
  basestyle/.style={colback=red!3, colframe=red!45!black, coltitle=black, colbacktitle=red!10},
  envstyle/.style={colback=yellow!6, colframe=yellow!55!black, coltitle=black, colbacktitle=yellow!18}
}

\newtcblisting{plist}[2][]{enhanced, listing only, sharp corners=south, boxrule=0.5pt,
  left=2pt, right=2pt, top=1pt, bottom=1pt, before skip=0pt, after skip=3pt,
  fonttitle=\sffamily\scriptsize\bfseries, title={#2},
  listing options={basicstyle=\ttfamily\tiny, breaklines=true, breakindent=0pt, columns=fullflexible}, #1}

\usepackage[T1]{fontenc}
\usepackage{amsmath, amssymb, amsthm}
\usepackage{mathtools}
\usepackage{bbm}
\usepackage{dsfont}

\usepackage{tabularx}
\definecolor{hadrow}{HTML}{E8EEF8}

\usepackage[dvipsnames,table]{xcolor}
\definecolor{darkpink}{RGB}{199,21,140}

\usepackage{multirow}
\usepackage{graphicx}

\newcommand{\sk}{\cellcolor{blue!5}}

\usepackage{calc}
\newsavebox\CBox

\definecolor{decreasecolor}{HTML}{B22222}  
\definecolor{increasecolor}{HTML}{1F4E79} 

\usepackage{tcolorbox}
\tcbuselibrary{skins}
\definecolor{promptbg}{HTML}{F5F5F5}
\definecolor{promptfr}{HTML}{9E9E9E}
\definecolor{prompttitle}{HTML}{616161}

\usepackage[table]{xcolor}
\definecolor{hstate}{HTML}{FBF3D9}
\definecolor{hbehav}{HTML}{F6E3A8}
\definecolor{hgray}{HTML}{5F5F5A}

\usepackage{graphicx}
\usepackage[format=hang]{subcaption}

\usepackage{booktabs, array, wrapfig}
\usepackage{pifont}

\usepackage{listings}
\usepackage{fancyvrb}
\usepackage{algorithm}
\usepackage{algpseudocode}
\fvset{fontsize=\small}

\definecolor{citecolor}{RGB}{0,102,204}
\definecolor{linkcolor}{RGB}{190,105,30}
\definecolor{urlcolor}{RGB}{199,21,133}

\usepackage{hyperref}
\hypersetup{colorlinks,linktoc=all}
\usepackage[all]{hypcap}
\hypersetup{citecolor=citecolor}
\hypersetup{linkcolor=linkcolor}
\hypersetup{urlcolor=urlcolor}
\usepackage[nameinlink,capitalise,noabbrev]{cleveref}
\creflabelformat{equation}{#2\textup{#1}#3} 
\crefname{observation}{Observation}{Observations}
\Crefname{observation}{Observation}{Observations}
\crefname{section}{Section}{Sections}
\Crefname{section}{Section}{Sections}
\crefname{subsection}{Section}{Sections}
\Crefname{subsection}{Section}{Sections}
\crefname{subsubsection}{Section}{Sections}
\Crefname{subsubsection}{Section}{Sections}
\crefname{appendix}{Appendix}{Appendices}
\Crefname{appendix}{Appendix}{Appendices}
\crefname{subappendix}{Appendix}{Appendices}
\Crefname{subappendix}{Appendix}{Appendices}
\crefname{subsubappendix}{Appendix}{Appendices}
\Crefname{subsubappendix}{Appendix}{Appendices}

\lstdefinestyle{mystyle}{
    commentstyle=\color{OliveGreen},
    numberstyle=\tiny\color{black!60},
    stringstyle=\color{BrickRed},
    basicstyle=\ttfamily\scriptsize,
    breakatwhitespace=false,
    breaklines=true,
    captionpos=b,
    keepspaces=true,
    numbers=none,
    numbersep=5pt,
    showspaces=false,
    showstringspaces=false,
    showtabs=false,
    tabsize=2
}
\usepackage{amsmath,amsfonts,bm}

\def\eqref#1{equation~\ref{#1}}

\def\1{\bm{1}}

\DeclareMathAlphabet{\mathsfit}{\encodingdefault}{\sfdefault}{m}{sl}
\SetMathAlphabet{\mathsfit}{bold}{\encodingdefault}{\sfdefault}{bx}{n}

\usepackage{lipsum}
\usepackage{todonotes}

\usepackage{amsthm}
\newtheorem{proposition}{Proposition}[section]
\newtheorem{corollary}{Corollary}[section]
\theoremstyle{definition}
\newtheorem{definition}{Definition}[section]
\newtheorem{observation}{Observation}[section]
\theoremstyle{plain}
\AddToHook{env/proposition/begin}{\crefalias{section}{proposition}}
\AddToHook{env/corollary/begin}{\crefalias{section}{corollary}}
\AddToHook{env/definition/begin}{\crefalias{section}{definition}}
\AddToHook{env/observation/begin}{\crefalias{section}{observation}}

\usepackage{hyperref}
\usepackage{url}
\usepackage{kotex}
\usepackage{tabularx}
\usepackage{booktabs, multirow}
\usepackage{array}
\usepackage{graphicx}
\usepackage{subcaption}

\hypersetup{
    colorlinks = true,
    citecolor = {cite_blue},
    linkcolor = {cite_purple},
    urlcolor = {cite_purple},
}

\reportnumber{}

\title{Harness-Aware Distillation for Small Language Model Agents}

\author[1]{Moonseok Choi}
\author[2]{Taehong Moon}
\author[1,$\dagger$]{Giung Nam}
\author[1,$\dagger$]{Juho Lee}

\makeatletter
\renewcommand{\AB@affilsepx}{,\protect\hspace{0.5em}\protect\Affilfont}
\makeatother

\affil[1]{KAIST AI}
\affil[2]{Independent researcher}

\correspondingauthor{%
$^\dagger$ Co-corresponding authors
\newline
Correspondence to:
{
\email{\{ms.choi, giung, juholee\}@kaist.ac.kr}}.%
}

\begin{abstract}
\textbf{Abstract:}
Language model agents are deployed with a harness, the software around the model that manages its context, tools, and feedback.
When such an agent is distilled into a smaller one, the harness stays in place, so the student mainly needs the teacher-specific abilities that the harness cannot provide, such as acting correctly on harness information.
Standard distillation, however, imitates the teacher's full outputs and treats the harness as part of the input.
We propose Harness-Aware Distillation (HAD), which focuses distillation on what the teacher adds beyond the harness.
HAD complements on-policy distillation with two components: an action preference that contrasts the same teacher's actions with and without the harness information, scored after the student's own reasoning, and a validity check that drops preference pairs whose preferred action contradicts the harness records.
We show that the contrast gives the student information that imitating the teacher alone cannot provide, and HAD needs no task rewards, success labels, or future information.
Across multiple long-horizon agent benchmarks and models, HAD outperforms on-policy distillation baselines with the same fixed harness.
Our analysis shows that HAD enters fewer unproductive loops and recovers from errors more often than the baselines, and suggests that it adaptively keeps learnable feedback in its weights while reading state information from the harness. The project page is available at \href{https://moon4sake.github.io/harness-aware-distillation}{link}.
\end{abstract}

\begin{document}

\maketitle

\section{Introduction}

\begin{wrapfigure}{r}{0.5\textwidth}
  \centering
  \vspace{-1.0em}
  \includegraphics[width=\linewidth]{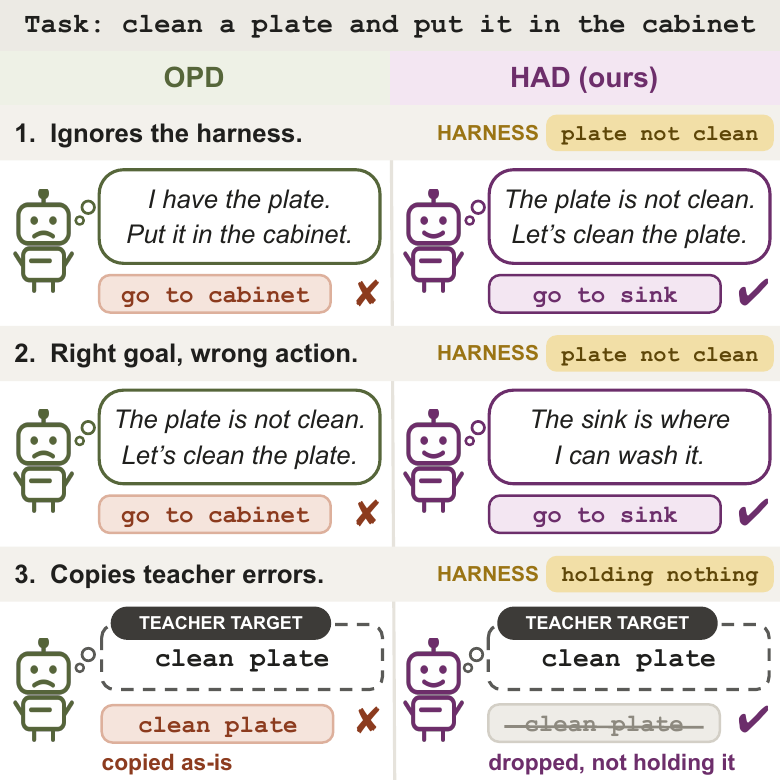}
  \vspace{-2.0em}
    \caption{OPD versus HAD on the three failures of OPD. HAD acts on the harness information in the first two rows and drops the invalid teacher target from its preference term in the last.}
\label{fig:difference}
\end{wrapfigure}

Language model agents solve long-horizon tasks through repeated interactions with an environment and its tools~\citep{yao2023react,yang2024sweagent}.
They are usually deployed inside a \emph{harness}, the software around the model that manages its context, tools, and feedback~\citep{wang2025openhands,lee2026metaharness}.
A harness can, for example, track the environment state, list the available actions, and report when the agent repeats an action without effect.
Changing only the harness, with the model fixed, can improve agent performance substantially~\citep{yang2024sweagent,lee2026metaharness}, so practical deployments assume a harness-equipped agent.

At the same time, the high cost of large models gives a strong incentive to distill their capabilities into smaller agents~\citep{belcak2025small,zeng2024agenttuning}.
On-policy distillation (OPD) is a common choice, as it provides teacher supervision at the states the student itself visits~\citep{agarwal2024gkd,lu2025onpolicy,tcod2026,sageopd2026}, but existing methods do not explicitly account for the harness.
Distilling a harness-equipped agent is a problem between two systems that share the same harness.
The harness stays fixed when the teacher is replaced by a smaller model, so the student does not need to learn what the harness already provides.
What it needs is the teacher's own ability to act on what the harness reports.
A natural approach is to run OPD with the same harness for the teacher and the student, yet this approach does not reliably transfer the ability.
As \cref{fig:motivation} shows, adding the harness to OPD changes the student's performance inconsistently across methods.
The agreement between the student's actions and the harness records, which we call harness utilization, also shows limited improvement and can even fall below that of the untrained student.

We attribute such limited gains to three failures in the supervision that OPD provides, illustrated in \cref{fig:difference}.
First, \emph{OPD leaves the teacher's use of the harness implicit}.
\Cref{tab:main} shows that part of the teacher's performance comes from the harness, and keeping this benefit requires the student to learn when harness information should change its decision.
OPD supervises every step in the same way, so it does not mark the few steps where the harness changes the teacher's action, and it never shows what the teacher would choose without the harness.
Second, \emph{OPD does not separate what the harness provides from what the teacher adds}.
The harness supplies information and feedback, but deciding how to act on them is an ability of the teacher that the harness cannot provide.
In the middle row of \cref{fig:difference}, the student reads that the plate is not clean, yet goes to the cabinet instead of the sink.
Third, \emph{OPD accepts teacher targets that contradict the harness records}.
Prior work shows that teachers can be less reliable at student-visited states~\citep{tcod2026,li2026guidedopd,sageopd2026}.
The harness records can reveal teacher errors, such as a target that requires an object the agent does not hold, but OPD does not check teacher targets against the harness records.

Our key insight is that the shared harness can itself guide distillation, as it shows which teacher decisions depend on it and its records reveal teacher targets that contradict observed facts.
We propose \textbf{Harness-Aware Distillation (HAD)}, which adds two components to response distillation, our term for on-policy distillation of the teacher's full response with the harness enabled, as illustrated in \cref{fig:schematic}.
\emph{Harness-aware action preference learning} queries the same teacher with and without the harness information and trains the student to prefer the action chosen with the harness, scored after the student's own reasoning.
This contrast is active only at states where the harness changes the teacher's action, and as we show in \cref{sec:had-theory}, it carries information that imitating the teacher's response alone does not provide.
\emph{Harness-aware action validity filtering} drops preference pairs whose preferred action contradicts the harness records, while response distillation is unchanged.
HAD uses no task rewards, success labels, or future information.

\begin{figure}[t]
    \centering
    \includegraphics[width=0.9\linewidth]{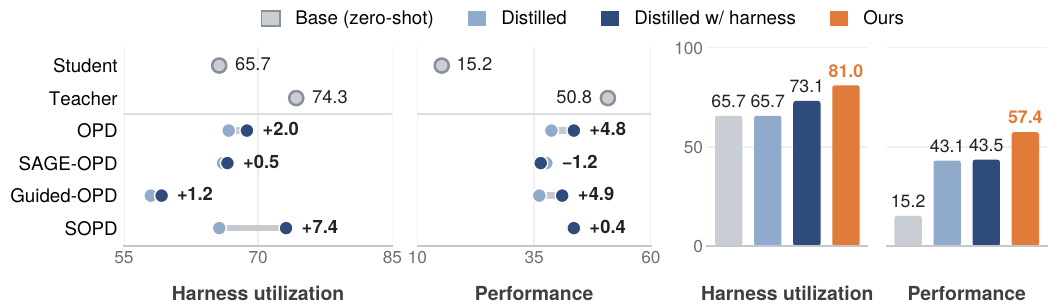}
\caption{Harness utilization and performance of distilled students on ALFWorld. \textbf{Left:} Each OPD method distilled without (light) and with (dark) the harness. \textbf{Right:} The zero-shot student, a distilled student trained without and with the harness, and HAD. Adding the harness to distillation improves harness utilization (65.7\%$\rightarrow$73.1\%) but not performance (43.1\%$\rightarrow$43.5\%), while HAD improves both (65.7\%$\rightarrow$81.0\%; 43.1\%$\rightarrow$57.4\%).}
    \label{fig:motivation}
    \vspace{-0.8em}
\end{figure}

Our contributions are as follows:
\begin{itemize}
    \setlength\itemsep{0pt}
    \item We view the distillation of a harness-equipped agent as distillation between two systems that share the harness, so the student should learn the teacher-specific abilities that the harness cannot provide.
    We show that on-policy distillation with the shared harness does not reliably teach the student what the teacher adds beyond the harness.
    \item We propose HAD, which queries the same teacher with and without the harness and teaches the student to prefer the action chosen with the harness.
    Since the teacher can still make mistakes, a validity check drops pairs whose preferred action contradicts the harness records.
    We show in \cref{sec:had-theory} that the contrast adds information that imitating the teacher alone cannot provide.
    \item HAD outperforms on-policy distillation baselines with the same harness across multiple long-horizon agent benchmarks and models.
    Our analysis in \cref{sec:analysis} shows that HAD learns to recover from stalls, unlike the baselines, and suggests that it keeps learnable feedback in its weights while reading state information from the harness.
\end{itemize}
\section{Related Work}

\paragraph{Distilling language model agents.}
Language model agents combine model predictions with external actions and feedback~\citep{yao2023react,shinn2023reflexion}.
Building on knowledge distillation~\citep{hinton2015distilling,kim2016sequence}, agent distillation transfers reasoning and tool-use behavior to smaller models~\citep{chen2023fireact,kang2025distilling}. Structured supervision further distinguishes between reasoning and action spans~\citep{liu2026structured}.
To reduce the mismatch between teacher-generated training data and student behavior, on-policy distillation supervises student-generated sequences using teacher feedback~\citep{agarwal2024gkd,gu2024minillm}. In multi-turn settings, this paradigm has been extended with step-level supervision~\citep{sopd2026}, temporal curricula~\citep{tcod2026}, and selective teacher intervention~\citep{sageopd2026}.

\paragraph{Learning with external context and agent harnesses.}
External context and agent harnesses can support inference and provide supervision during training. Context distillation and harness self-distillation aim to internalize contextual knowledge or harness-augmented reasoning into model weights~\citep{ye2026onpolicycontext,zhao2026ophsd}. Other work studies the interplay between harness design and post-training, including training compact agents to operate across varying harness configurations~\citep{kim2026interplay}.
External guidance can also shape the supervision signal itself: for example, rollout returns can determine the direction of distillation between skill-conditioned and unconditioned views~\citep{tu2026ucob}.
Our focus is instead on explicitly supervising how smaller agents use information provided by the harness, while retaining the harness at deployment time.

\paragraph{Contrastive supervision and preference-based distillation.}
Preference learning trains models to favor one candidate over another, using either a reference-based objective or a reference-free objective combined with imitation~\citep{rafailov2023dpo,xu2024contrastive}. These approaches also differ in how they construct preference pairs. Prior work constructs synthetic pairs through contrasting prompts~\citep{yang2024rlcd}, compares answers generated with and without retrieval using answer correctness~\citep{yan2025rpo}, or forms preferences between teacher and student generations~\citep{li2024dpkd,yu2026pbsd}.
Our method constructs action preferences from the same teacher with and without harness information at student-visited states. We treat the harness-conditioned action as preferred and exclude pairs when that action contradicts harness records. This supervision requires neither task rewards nor success labels.

\section{Method}\label{sec:method}

\subsection{Harness-Aware Distillation}\label{sec:had}

We consider a multi-turn agent interacting with an environment over a horizon $T$.
For a task $u\sim\mathcal{D}$, an external harness augments each environment observation $o_t$ with supplementary information $h_t$, computed deterministically from the environment state and interaction history.
HAD makes no assumption about what $h_t$ contains, and only requires that it is computed deterministically from the environment state and the interaction history.
Writing $\tilde{o}_t=(o_t,h_t)$, we define the \emph{harness-induced history} $\tilde{x}_t$ and its \emph{stripped history} $x_t$ as
\begin{equation}
\tilde{x}_t=(u,\tilde{o}_0,\tilde{y}_0,\ldots,\tilde{y}_{t-1},\tilde{o}_t),
\qquad
x_t=(u,o_0,\tilde{y}_0,\ldots,\tilde{y}_{t-1},o_t).
\label{eq:had-histories}
\end{equation}
The agent conditions on $\tilde{x}_t$ to generate a response $\tilde{y}_t=(\tilde{z}_t,\tilde{a}_t)$ consisting of a reasoning trace and an action.
The environment executes $\tilde{a}_t$ and returns the next observation $o_{t+1}$.
Stripping removes only $h_0,\ldots,h_t$, not their influence on the realized observations and responses; $x_t$ is therefore not a separately generated harness-free trajectory.

\paragraph{Harness-equipped response distillation.}
Let $\pi_\phi$ be a teacher and $\pi_\theta$ a student.
With the harness enabled for both the student and the teacher, we collect student rollouts matching deployment and query the teacher at each visited history.
\begin{equation}
    \tilde{y}_t^\theta=(\tilde{z}_t^\theta,\tilde{a}_t^\theta)
    \sim\pi_\theta(\cdot\mid\tilde{x}_t^\theta),
    \qquad
    \tilde{y}_t^\phi=(\tilde{z}_t^\phi,\tilde{a}_t^\phi)
    \sim\pi_\phi(\cdot\mid\tilde{x}_t^\theta).
\end{equation}
We distill the teacher's complete response using its mean token negative log-likelihood:
\begin{equation}
    \ell_{\mathrm{dist},t}(\theta)
    =-\frac{1}{|\tilde{y}_t^\phi|}
    \sum_{j=1}^{|\tilde{y}_t^\phi|}
    \log\pi_\theta\!\left(
        \tilde{y}_{t,j}^\phi
        \mid\tilde{x}_t^\theta,\tilde{y}_{t,<j}^\phi
    \right).
    \label{eq:had-response-distillation}
\end{equation}

\paragraph{Harness-aware action preference learning.}
Distilling at student-visited states addresses the distribution shift of training on teacher trajectories, but does not explicitly supervise harness-induced behavioral differences.
We therefore query the same teacher on the stripped history:
\begin{equation}
    y_t^\phi=(z_t^\phi,a_t^\phi)
    \sim\pi_\phi(\cdot\mid x_t^\theta).
\end{equation}
Given the visited history, the two teacher responses, with and without the harness, are independent of each other and of the student's current reasoning.
Using access to the explicit harness information as a quality proxy, we prefer $\smash{\tilde{a}_t^\phi}$ over $\smash{a_t^\phi}$ and score both under the same student reasoning $\smash{\tilde{z}_t^\theta}$: 
\begin{align}
    \Delta_t(\theta)
    =\frac{1}{|\tilde{a}_t^\phi|}
    \log\pi_\theta\!\left(
        \tilde{a}_t^\phi\mid\tilde{x}_t^\theta,\tilde{z}_t^\theta
    \right)
    -\frac{1}{|a_t^\phi|}
    \log\pi_\theta\!\left(
        a_t^\phi\mid\tilde{x}_t^\theta,\tilde{z}_t^\theta
    \right).
\end{align}
Each score averages log-likelihood over action tokens only, excluding the reasoning prefix.
We minimize the \textit{reference-free} pairwise logistic loss (see \cref{app:had-reference-choice} for its regularization bias):
\begin{equation}
    \ell_{\mathrm{pref},t}(\theta)
    =-\log\sigma\!\left(\beta\Delta_t(\theta)\right),
    \label{eq:had-action-preference}
\end{equation}
where $\sigma$ is the sigmoid and $\beta>0$ controls the margin scale.
Unlike response distillation, this term supervises actions under student rather than teacher reasoning.

\paragraph{Harness-aware action validity filtering.}
We set $m_t=V(\tilde{a}_t^\phi\mid\tilde{x}_t^\theta)$, where the validity checker $V(a\mid\tilde{x})\in\{0,1\}$ evaluates actions using the harness-induced history.
Specifically, we judge the action invalid if it conflicts with the harness-induced environment observation.
The mask applies only to the harness-aware preference learning term, leaving the harness-equipped distillation term unchanged.
See \cref{app:had-implementation} for details on how $V$ is implemented in practical settings.

\paragraph{Final objective.}
Let $\nu$ denote the distribution of collected training records.
Our proposed Harness-Aware Distillation (HAD) combines response distillation with validity-filtered action preferences:
\begin{equation}
    \mathcal L_{\mathrm{HAD}}(\theta;\lambda)
    =\mathbb E_\nu\!\left[
        \ell_{\mathrm{dist},t}(\theta)
        +\lambda m_t\ell_{\mathrm{pref},t}(\theta)
    \right].
    \label{eq:had-objective}
\end{equation}
We adjust $\lambda$ at each update using minibatch gradient-norm balancing of the two loss terms (\cref{app:had-implementation}).
The ablations in \cref{sec:ablation} show gains over harness-equipped distillation alone ($\lambda=0$), unfiltered preferences ($m_t\equiv1$), full-response preferences, positive-only supervision, and a teacher reference model.
Full-response preferences compare $\smash{\tilde{y}_t^\phi}$ with $\smash{y_t^\phi}$ under the same augmented history, rather than comparing actions under student reasoning.
Positive-only supervision imitates $\smash{\tilde{a}_t^\phi}$ under the student reasoning prefix without contrasting it with $\smash{a_t^\phi}$.
See \cref{app:had-ablation-variants} for more detailed definitions; corresponding ablation results are provided in \cref{sec:ablation,app:ablation-extra}.

\begin{figure}[t]
    \centering
    \includegraphics[width=0.8\linewidth]{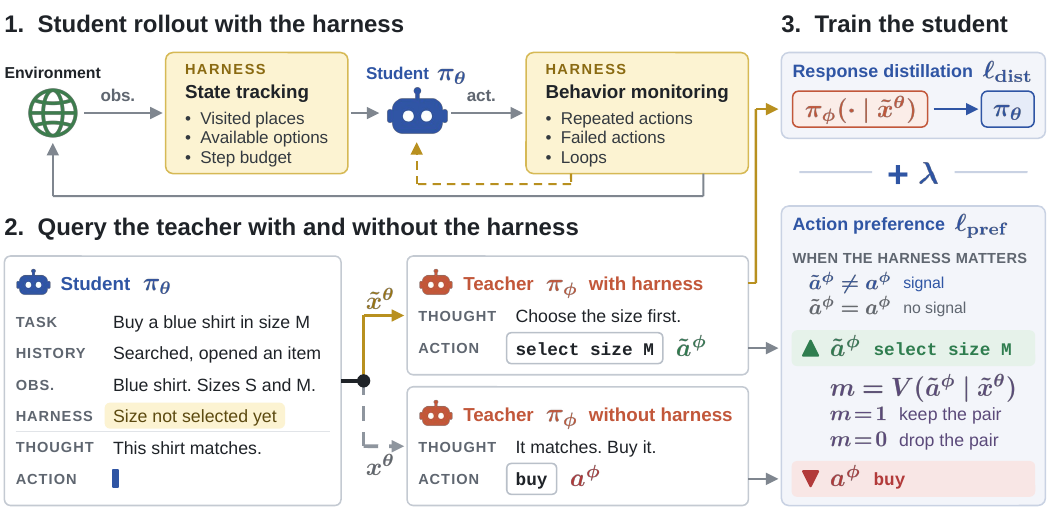}
\caption{Overview of HAD. (1) The student acts through the harness, which tracks the environment state and monitors the student's behavior. (2) At each visited state, the same teacher is queried with and without the harness information. (3) The student imitates the teacher's response given the harness and learns to prefer the action chosen with the harness over the action chosen without it, which gives a signal only when the two actions differ. The validity mask drops pairs whose preferred action contradicts the harness records.}
    \label{fig:schematic}
\end{figure}

\subsection{What Harness Awareness Adds to Distillation}\label{sec:had-theory}

Our proposed harness-aware preference learning makes \textit{harness awareness} explicit by contrasting the same teacher's actions with and without explicit harness information.
For a fixed harness-induced history $\smash{\tilde{x}^\theta}$ and its stripped history $\smash{x^\theta}$, let $\smash{P_+(a) = \sum_{z} \pi_{\phi}(z,a \mid \tilde{x}^\theta)}$ and $\smash{P_-(a) = \sum_{z} \pi_{\phi}(z,a \mid x^\theta)}$ be the teacher's action marginals.
Here, $+$ denotes the teacher with explicit harness information, while $-$ denotes the same teacher with that information stripped from the history.
Assuming that harness information generally helps the teacher choose better actions, we form each preference pair by labeling the harness-equipped action from $P_+$ positive and the stripped-history action from $P_-$ negative.
The analysis below characterizes the auxiliary preference term; formal assumptions and derivations appear in \cref{app:had-preference-adds} and \cref{app:had-filtering-adds}.


\paragraph{\textit{What supervision is added?}}
Response distillation already supervises harness-equipped behavior, but it does not distinguish actions that are generally likely under the teacher from those that become more likely specifically because of the harness.
Our additional supervision reveals how the teacher's action probabilities differ without the harness information.
For independent draws from strictly positive laws on a common finite action set, the isolated unfiltered preference loss $\ell_{\mathrm{pref,t}}$ has optimal unrestricted scores $\log(P_+(a)/P_-(a))+\kappa$, with arbitrary constant $\kappa$ as in \cref{prop:had-preference}.
Thus, our harness-aware preference term does not simply reinforce actions with high $P_+(a)$; it emphasizes actions with high $P_+(a)$ relative to $P_-(a)$.
Response distillation alone cannot identify this relative change, since its risk depends only on the positive response law  as mentioned in \cref{obs:had-additional-contrast}.

\paragraph{\textit{Why filter positive preference labels?}}
Harness access does not guarantee that the positive action $a \sim P_+$ is valid.
When it is impossible to know which action is correct, the best we can do is exclude actions that are clearly invalid because they conflict with the harness-induced environment observation.
We thus use the validity checker $V$, with $V(a\mid\tilde{x}^\theta)=1$ denoting acceptance.
If the positive action is rejected, we drop only its preference loss, leaving response distillation unchanged.
With independent teacher sampling and nonzero acceptance, the masked risk is the acceptance probability times the preference risk with $P_+$ conditioned on acceptance and $P_-$ unchanged (cf. \cref{obs:had-filtering}).
However, filtering cannot remove cases where the positive action is valid yet incorrect; this limitation is inherent to any distillation framework that follows the incomplete teacher.

\subsection{Why Action-Only Preference under Student Reasoning}
\label{sec:had-action-design}

Our preference term is designed with two considerations: what to compare and under which reasoning context to compare it.
We motivate these choices in the following paragraphs, with further formal analysis in \cref{app:had-joint-preferences,app:had-prefix-supervision}.

\paragraph{\textit{Why action-only?}}
Our goal is to supervise the action difference induced by the harness, rather than differences between complete teacher responses.
When the teacher is queried with and without the harness, its reasoning can differ substantially in content or style even when the resulting action is the same.
Comparing complete responses can therefore spend much of the preference signal on distinguishing such reasoning differences, rather than focusing on the action change induced by the harness.
This effect is further amplified because reasoning traces are typically much longer than actions, allowing them to dominate the full-response comparison and weaken the action-level signal.
Indeed, our empirical ablation in \cref{tab:app-design} shows that the full-response variant performs worse.
\cref{app:had-joint-preferences} further formalizes this distinction: full-response preferences combine the desired action contrast with an additional teacher-reasoning contrast, which can alter the induced action ordering even when the teacher action marginals are unchanged.

\paragraph{\textit{Why student-generated reasoning?}}
Having isolated the action-level comparison, we score the teacher actions with and without the harness under the student’s own sampled reasoning.
This matches the deployed student’s decision context: at inference time, it chooses an action after its own reasoning, not after the reasoning underlying either teacher action.
Accordingly, the preference term does not condition action supervision on the teacher’s no-harness reasoning; the harness-equipped teacher’s reasoning is instead learned separately through response distillation.
Thus, the auxiliary preference term focuses on action choices under the student’s own reasoning, while response distillation provides supervision for the desired teacher reasoning.
\cref{app:had-prefix-supervision} further formalizes that student-prefix conditioning preserves \(P_+\) and \(P_-\) at student-visited reasoning contexts, while response distillation complements the auxiliary objective with reasoning-dependent supervision.

\section{Experiment}\label{sec:experiment}
\subsection{Experimental setup}
\label{sec:exp-setup}

\paragraph{Models and datasets.}
We evaluate on ALFWorld~\citep{shridhar2021alfworld}, WebShop~\citep{yao2022webshop}, and ScienceWorld~\citep{wang2022scienceworld}.
Using models from the Qwen3 family~\citep{yang2025qwen3}, we distill 1.7B from the 8B teacher on ALFWorld, and 0.6B from the 30B-A3B teacher on WebShop and ScienceWorld.
To assess generalization across model families, we additionally distill Gemma-4-E2B from Gemma-4-12B on ALFWorld~\citep{gemmateam2026gemma4}.
All evaluations use temperature 0.4 with thinking disabled; \cref{app:env} gives further details.


\begin{wraptable}{r}{0.38\textwidth}
\vspace{-1em}
\centering
\small
\caption{Harness components.\vspace{-0.5em}}
\label{tab:harness}
\setlength{\aboverulesep}{0pt}
\setlength{\belowrulesep}{0pt}
\setlength{\extrarowheight}{0.15em}
\setlength{\tabcolsep}{0.5em}
\renewcommand{\arraystretch}{1.1}
\begin{tabular}{ll}
\toprule
\textbf{Component} & \textbf{Computed from} \\
\midrule
\rowcolor{hstate}\multicolumn{2}{c}{\textsc{State Tracking}} \\
Available actions & \textcolor{hgray}{Environment action list} \\
Situation & \textcolor{hgray}{Location and inventory} \\
History & \textcolor{hgray}{Visit history} \\
Budget & \textcolor{hgray}{Step counter} \\
\midrule
\rowcolor{hstate}\multicolumn{2}{c}{\textsc{Behavior Monitoring}} \\
Failure & \textcolor{hgray}{Actions at the same state} \\
Stall & \textcolor{hgray}{Recent states} \\
\bottomrule
\end{tabular}
\vspace{-2em}
\end{wraptable}

\paragraph{Harness.}
We set the harness to be task-agnostic, sharing the same code and templates across environments without task-specific knowledge such as subgoals, and it only adds text to the observation without selecting or blocking actions.
Its components follow functions common in agent harnesses~\citep{wang2025openhands} and fall into state tracking and behavior monitoring, as listed in \cref{tab:harness}.
State tracking is shown at every step, and behavior monitoring only when its condition fires.
Every method uses the same harness for training and evaluation, and \cref{app:harness} gives the templates, thresholds, validity check, and example pages.

\paragraph{Metrics.}
We report the success rate (SR) on all environments, the percentage of episodes that complete the task within the step limit, and the score on WebShop and ScienceWorld, given by the mean final reward scaled to 100 and the mean final task score, respectively.
We also report the average number of turns in successful episodes, and harness utilization (HU), which measures how often the agent's next action agrees with what the harness reports.
For each harness component, we compute the fraction of steps where the next action is consistent with its report and average these fractions over the components.
For example, the next action should be among the available actions and should differ from an action reported as failed.
We use this metric only as a diagnostic of how the agent responds to the harness, and \cref{app:hu} gives the full rules and counting procedure.

\paragraph{Baselines.}
We compare HAD with four OPD methods, all trained and evaluated with the same harness.
OPD~\citep{agarwal2024gkd} matches the teacher with the token-level reverse KL divergence on the student's own responses, SAGE-OPD~\citep{sageopd2026} reweights each turn by how strongly the teacher judges that it needs correction, Guided-OPD~\citep{li2026guidedopd} lets the teacher take over a decreasing fraction of turns, and SOPD~\citep{sopd2026} imitates a full step written by the teacher at each state the student visits.
All methods share the same training loop, task pool, and number of updates, and we also report the zero-shot teacher and student.
\Cref{app:implementation_details} gives implementation details and differences from the original papers.

\paragraph{Hyperparameter setting.}
All methods use full fine-tuning with AdamW, a cosine learning rate schedule with 15 warmup updates, gradient clipping at 1.0, weight decay 0.01, and an effective batch size of 60.
The learning rate is $10^{-5}$ for all methods.
Each round rolls out 16 tasks with the current student, and each collected step is used once within two policy versions.
We train for 250 updates on ALFWorld and 150 updates on WebShop and ScienceWorld.
For HAD, we use $\beta=0.5$ and set $\lambda$ at each update so that the gradient norm of the preference term is $\rho=0.5$ times that of response distillation.
\Cref{app:had-implementation} and \Cref{app:implementation_details} list all the details.

\subsection{Main experiment}

\newcommand{\hcell}[2]{#1\,{\scriptsize\textcolor{gray}{/\,#2}}}
\newcommand{\hdcell}[2]{\textit{#1}\,{\scriptsize\textcolor{gray}{/\,\textit{#2}}}}
\newcommand{\hbcell}[2]{\textbf{#1}\,{\scriptsize\textcolor{gray}{/\,#2}}}

\begin{table*}[t]
\centering
\small
\caption{Main results. Each cell reports performance \textbf{with harness} \textcolor{gray!90}{\scriptsize/\,without harness}. All distillation methods are trained with the harness, and the teacher and the student are evaluated zero-shot. Bold marks the best student in each column.\vspace{-0.5em}}
\label{tab:main}
\setlength{\aboverulesep}{0pt}
\setlength{\belowrulesep}{0pt}
\setlength{\extrarowheight}{0.15em}
\setlength{\tabcolsep}{0.5em}
\renewcommand{\arraystretch}{1.1}
\resizebox{\textwidth}{!}{\begin{tabular}{l*{6}{>{\centering\arraybackslash}p{2.0cm}}}
    \toprule
    & \multicolumn{2}{c}{ALFWorld {\footnotesize (1.7B $\leftarrow$ 8B)}}
    & \multicolumn{2}{c}{WebShop {\footnotesize (0.6B $\leftarrow$ 30B)}}
    & \multicolumn{2}{c}{ScienceWorld {\footnotesize (0.6B $\leftarrow$ 30B)}} \\
    \cmidrule(lr){2-3} \cmidrule(lr){4-5} \cmidrule(lr){6-7}
    Method & Unseen SR & Seen SR & Score & SR & Score & SR \\
    \midrule
    \textit{Teacher}
        & \hcell{51.49}{38.06} & \hcell{50.00}{37.86}
        & \hcell{49.96}{41.41} & \hcell{18.80}{13.40}
        & \hcell{\phantom{$-$}20.26}{\phantom{$-$}18.47} & \hcell{20.13}{15.44} \\
    \textit{Student}
        & \hcell{\phantom{0}9.70}{\phantom{0}2.99} & \hcell{20.71}{\phantom{0}6.43}
        & \hcell{20.95}{\phantom{0}4.29} & \hcell{\phantom{0}9.80}{\phantom{0}1.40}
        & \hcell{$-$16.47}{$-$16.97} & \hcell{\phantom{0}3.36}{\phantom{0}3.36} \\
    \midrule
    OPD
        & \hcell{47.01}{31.34} & \hcell{40.00}{31.43}
        & \hcell{41.59}{39.05} & \hcell{19.00}{16.60}
        & \hcell{\phantom{0}$-$0.82}{\phantom{0}$-$0.64} & \hcell{\phantom{0}8.05}{\phantom{0}7.38} \\
    SAGE-OPD
        & \hcell{31.34}{25.37} & \hcell{41.43}{34.29}
        & \hcell{42.21}{39.15} & \hcell{17.00}{14.00}
        & \hcell{\phantom{0$-$}0.77}{\phantom{0$-$}1.41} & \hcell{\phantom{0}6.04}{\phantom{0}6.71} \\
    Guided-OPD
        & \hcell{46.27}{33.58} & \hcell{35.71}{27.14}
        & \hcell{43.89}{40.09} & \hcell{16.60}{14.40}
        & \hcell{\phantom{0}$-$2.52}{\phantom{0}$-$4.17} & \hcell{\phantom{0}6.04}{\phantom{0}6.04} \\
    SOPD
        & \hcell{45.52}{26.87} & \hcell{41.43}{27.86}
        & \hcell{39.97}{37.64} & \hcell{13.80}{13.60}
        & \hcell{\phantom{0}$-$8.45}{\phantom{0}$-$8.00} & \hcell{\phantom{0}6.71}{\phantom{0}5.37} \\
    \sk \textbf{HAD (ours)}
        & \sk\hbcell{63.43}{40.30} & \sk\hbcell{51.43}{37.86}
        & \sk\hbcell{47.19}{41.63} & \sk\hbcell{19.60}{15.00}
        & \sk\hbcell{\phantom{0$-$}4.17}{\phantom{0$-$}4.37} & \sk\hbcell{\phantom{0}9.40}{\phantom{0}8.72} \\
    \bottomrule
\end{tabular}}
\end{table*}

\begin{figure}[t]
    \centering
    \includegraphics[width=\textwidth]{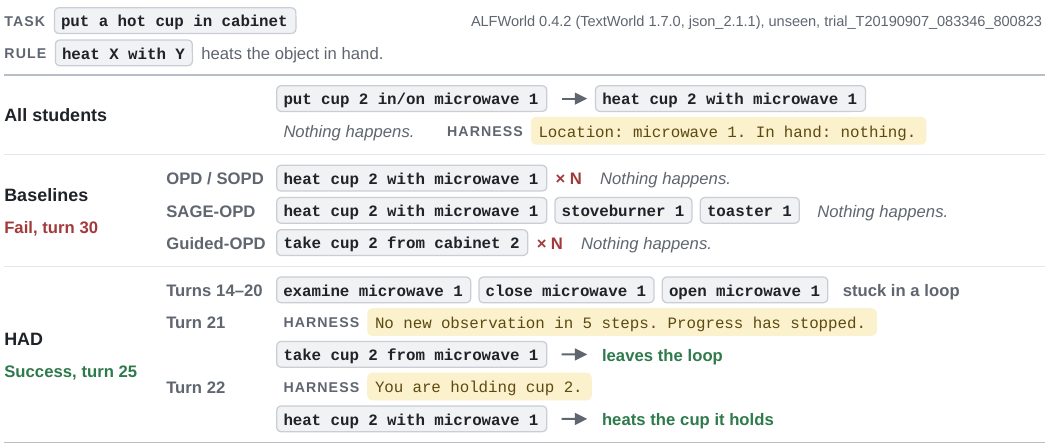}
    \vspace{-1.5em}
    \caption{An unseen ALFWorld task where heating works only on the object in hand. The baselines loop on actions until the step limit, while HAD leaves its loop after the stall report and heats the cup once the harness reports that it holds it.}
    \label{fig:qualitative}
\end{figure}

\paragraph{Quantitative analysis.}
HAD achieves the best performance in every environment in \cref{tab:main}.
With the harness, it reaches 63.4\% on unseen and 51.4\% on seen ALFWorld tasks, against 47.0\% and 41.4\% for the best baseline, and it even exceeds its 8B teacher.
The harness also matters more for HAD than for the baselines, consistent with our design that trains the student on the decisions the harness changes.
Removing the harness at evaluation lowers HAD from 63.4\% to 40.3\% on unseen tasks, the largest drop among all methods, yet HAD still outperforms the majority of baselines.
\Cref{app:hu,app:turns} further show that its actions are the most consistent with the harness and that it takes a similar number of turns in successful episodes as the other methods.

\paragraph{Qualitative analysis.}
\Cref{fig:qualitative} shows an unseen ALFWorld task that asks the agent to heat a cup and put it in a cabinet, where heating works only on the object the agent holds.
Every student puts the cup in the microwave and then tries to heat it with an empty hand, and the baselines keep repeating actions with no effect until the step limit.
HAD falls into a similar loop but leaves it with the help of the harness.
After the harness flags that five steps produced nothing new, HAD takes the cup back, and once the harness shows that it holds the cup, HAD heats it and completes the task.
\Cref{app:prompts} gives additional examples.

\subsection{Analysis}
\label{sec:analysis}

\begin{table}[t]
\centering
\small
\setlength{\aboverulesep}{0pt}
\setlength{\belowrulesep}{0pt}
\setlength{\extrarowheight}{0.15em}
\begin{minipage}[t]{0.49\linewidth}\vspace{0pt}
\centering
\caption{Prevention (episodes without stalls) and recovery (stalls escaped) rates (\%) on ALFWorld.\vspace{-0.5em}}
\label{tab:recovery}
\renewcommand{\arraystretch}{1.0}
\begin{tabularx}{\linewidth}{l*{2}{>{\centering\arraybackslash}X}}
\toprule
 & Prevention & Recovery \\
\midrule
Teacher & 64.6 & 47.0 \\
Student & 20.4 & 46.8 \\
\midrule
OPD & 64.2 & 42.9 \\
SAGE-OPD & 60.2 & 32.4 \\
Guided-OPD & 60.6 & 18.8 \\
SOPD & 62.0 & 50.0 \\
\rowcolor{hadrow} \textbf{HAD (ours)} & \bf{75.9} & \bf{59.7} \\
\bottomrule
\end{tabularx}
\end{minipage}\hfill
\begin{minipage}[t]{0.49\linewidth}\vspace{0pt}
\centering
\caption{Success rate of HAD on ALFWorld with one harness category hidden or harness-rejected actions blocked.\vspace{-0.5em}}
\label{tab:removal}
\renewcommand{\arraystretch}{1.14}
\begin{tabularx}{\linewidth}{l*{2}{>{\centering\arraybackslash}X}}
\toprule
Harness at evaluation & Unseen & Seen \\
\midrule
Full harness & 63.43 & 51.43 \\
\midrule
\textit{Hidden} & & \\
\quad Behavior monitoring & 55.97 & 45.00 \\
\quad State tracking & 37.31 & 40.71 \\
\midrule
\textit{Blocked} & & \\
\quad Rejected actions & 55.22 & 46.43 \\
\bottomrule
\end{tabularx}
\end{minipage}
\end{table}
\paragraph{Error trajectory recovery.}
Long-horizon agents often fall into loops of actions that have no effect, and escaping such error states is a known difficulty~\citep{wang2025openhands,sageopd2026}.
We ask whether HAD avoids these states, which we call prevention, and whether it leaves a stall once stuck, which we call recovery, counting a stall as escaped when the agent reaches a new state.
\emph{HAD improves both prevention and recovery over every other model}, as shown in \cref{tab:recovery}.
It avoids stalls and repeated failed actions in 75.9\% of its episodes, against at most 64.6\% for the other models.
Standard distillation does not teach recovery, as no baseline clearly improves over the untrained student at 46.8\%, whereas HAD escapes 59.7\% of its stalls.
Every stall we count is flagged by the harness, which suggests that HAD learns to act on the harness to leave a stall, and the states where the two teacher actions differ are spread over the whole episode, as shown in \cref{app:disagreement}.

\paragraph{Harness internalization.}
A harness is internalized when the model learns to produce, from its own weights, the behavior that the harness would otherwise supply, so that it acts correctly even when that part of the harness is removed~\citep{zhao2026ophsd}.
HAD is not designed to internalize the harness, but our results suggest that it may internalize part of the harness while still relying on the rest.
Without the harness, HAD reaches 40.3\% on unseen and 37.9\% on seen tasks in \cref{tab:main}, above every baseline without the harness and on par with or above the teacher without the harness.
\Cref{tab:removal} shows that this reliance differs across harness categories.
Hiding state tracking lowers the success rate of HAD from 63.4\% to 37.3\% on unseen tasks and from 51.4\% to 40.7\% on seen tasks, whereas hiding behavior monitoring lowers it to 56.0\% and 45.0\%, although behavior monitoring also appears on fewer steps.
Blocking the actions that the harness rejects also lowers the success rate on both splits, which suggests that the harness is more useful to HAD as information to act on than as a hard constraint.
These results are consistent with HAD internalizing part of the feedback from behavior monitoring, while it still reads state information, such as the places visited and the objects held, from the harness.
State tracking thus appears more essential at deployment.

\subsection{Cross-model generalization}
\label{sec:generalization}

\begin{wraptable}{r}{0.44\textwidth}
    \vspace{-1.5em}
    \centering
    \small
    \caption{Gemma-4 results.\vspace{-0.5em}}
    \label{tab:gemma}
    \setlength{\aboverulesep}{0pt}
    \setlength{\belowrulesep}{0pt}
    \setlength{\extrarowheight}{0.15em}
    \renewcommand{\arraystretch}{1.1}
    \begin{tabularx}{\linewidth}{l*{2}{>{\centering\arraybackslash}X}}
    \toprule
    Method & Unseen SR & Seen SR \\
    \midrule
    \textit{Teacher} & \hcell{51.49}{43.28} & \hcell{49.29}{42.14} \\
    \textit{Student} & \hcell{14.93}{8.21} & \hcell{18.57}{15.71} \\
    \midrule
    OPD & \hcell{26.87}{18.66} & \hcell{31.43}{29.29} \\
    SAGE-OPD & \hcell{24.63}{20.90} & \hcell{28.57}{27.14} \\
    Guided-OPD & \hcell{29.10}{18.66} & \hcell{32.86}{24.29} \\
    SOPD & \hcell{32.09}{29.85} & \hcell{34.29}{25.71} \\
    \sk \textbf{HAD (ours)} & \sk \hbcell{37.31}{23.88} & \sk \hbcell{36.43}{30.00} \\
    \bottomrule
    \end{tabularx}
    \vspace{-0.8em}
\end{wraptable}

\Cref{tab:gemma} repeats the ALFWorld evaluation with the Gemma-4 family; E2B students are distilled from the 12B teacher, under the same recipe and harness.
HAD again achieves the highest success rate among the students, with 37.3\% on unseen and 36.4\% on seen tasks against 32.1\% and 34.3\% for the strongest baseline.
Its harness utilization also exceeds even that of the teacher, as detailed in \cref{app:gemma}.
This agreement across model families suggests that the gain of HAD comes from the harness contrast rather than from properties of a specific model family.

\subsection{Ablation study}
\label{sec:ablation}

\newcommand{\cmark}{\ding{51}}

\setlength{\intextsep}{5pt}
\setlength{\columnsep}{10pt}
\setlength{\abovecaptionskip}{2pt}
\setlength{\belowcaptionskip}{2pt}

\begin{wraptable}{r}{0.45\linewidth}
    \vspace{-0.8em}
    \centering
    \small
    \caption{Component ablations.}
    \label{tab:had-ablation}
    \setlength{\aboverulesep}{0pt}
    \setlength{\belowrulesep}{0pt}
    \setlength{\extrarowheight}{0.15em}
    \renewcommand{\arraystretch}{1.0}
    \begin{tabularx}{\linewidth}{X*{2}{>{\centering\arraybackslash}p{0.9cm}}}
    \toprule
    Variant & SR & HU \\
    \midrule
    \rowcolor{hadrow} \textbf{HAD (ours)} & \bf{57.4} & \bf{81.0} \\
    \midrule
    w/o validity ($m_t\equiv1$) & 50.8 & 78.9 \\
    w/o negative (pos-only) & 52.4 & 74.4 \\
    w/o preference ($\lambda=0$) & 43.5 & 73.1 \\
    \bottomrule
    \end{tabularx}
\end{wraptable}


\paragraph{Component ablation.}
\Cref{tab:had-ablation} removes each component of HAD in turn on ALFWorld.
Removing the validity check ($m_{t}\equiv1$) lowers the SR from 57.4\% to 50.8\%, since teacher actions that contradict the harness records then enter the preference loss.
Removing the negative, so that the student only imitates the action chosen with the harness after its own reasoning, lowers the SR to 52.4\% and harness utilization from 81.0\% to 74.4\%, close to the 73.1\% of response distillation alone.
The contrast with the action chosen without the harness therefore accounts for most of the gain in harness utilization (HU), while the validity check mainly protects task performance.
Removing the preference term leaves response distillation alone ($\lambda=0$), and the SR drops to 43.5\%.
Each component thus contributes to HAD in a complementary manner.

\paragraph{Teacher budget.}
\begin{wraptable}{r}{0.45\linewidth}
    \vspace{-0.8em}
    \centering
    \small
    \caption{Teacher budget ablations.}
    \label{tab:budget}
    \setlength{\aboverulesep}{0pt}
    \setlength{\belowrulesep}{0pt}
    \setlength{\extrarowheight}{0.15em}
    \renewcommand{\arraystretch}{1.0}
    \begin{tabularx}{\linewidth}{l*{2}{>{\centering\arraybackslash}X}}
    \toprule
    Method & Unseen SR & Seen SR \\
    \midrule
    OPD & 47.0 & 40.0 \\
    OPD, 2$\times$ & 47.8 & 45.7 \\
    \midrule
    SOPD & 45.5 & 41.4 \\
    SOPD, 2$\times$ & 50.8 & 42.9 \\
    \midrule
    \rowcolor{hadrow} \textbf{HAD (ours)} & \bf{63.4} & \bf{51.4} \\
    \bottomrule
    \end{tabularx}
\end{wraptable}
In the main results, every method is trained for the same number of updates, but HAD queries the teacher twice per state, once with and once without the harness, while the baselines query it once.
To separate the effect of this extra teacher budget, \cref{tab:budget} also trains OPD and SOPD with twice the teacher budget.
Each baseline sees twice as many student trajectories, so that the teacher is queried as often as in HAD, which also gives the baselines twice as many updates and favors them in this comparison.
Even with this advantage, the success rate averaged over seen and unseen tasks only rises from 43.5\% to 46.7\% for OPD and from 43.5\% to 46.8\% for SOPD, far below the 57.4\% of HAD.
The gain of HAD is therefore not explained by the additional teacher queries.

\paragraph{Design choices.}
HAD scores only the action after the student's own reasoning and uses no reference model, following \cref{sec:had-action-design}.
As shown in \cref{tab:app-design}, both alternatives to these choices are less effective than HAD (see \cref{app:ablation-extra}).
Comparing full responses instead of actions lets the loss drop by telling apart the teacher's reasoning styles with and without the harness, and the long reasoning dilutes the signal from the short action.
Scoring only the action after the student's own reasoning avoids both problems and matches how the student chooses its action at deployment.
Using the teacher as a reference model does not help either, since the teacher's probabilities after the student's reasoning are not a reliable prior at the states the student visits.
In \cref{tab:app-rho}, the success rate peaks at $\rho=0.5$, which suggests that a smaller weight gives too little signal from the harness contrast while a larger weight may pull the small student away from the task.
\section{Conclusion}
\label{sec:conclusion}

We have presented Harness-Aware Distillation (HAD), a distillation framework for small language model agents that are deployed with a harness.
HAD starts from the view that distilling such an agent is a problem between two systems that share the same harness, so the student mainly needs the teacher's ability to act on what the harness reports.
It complements on-policy distillation with an action preference that contrasts the same teacher's actions with and without the harness information, and with a validity check that drops pairs whose preferred action contradicts the harness records.
Across multiple long-horizon agent tasks, our method achieves the best performance among on-policy distillation methods equipped with the same harness.

\paragraph{Future work and limitations.}
We evaluated HAD with students of up to 2B parameters and teachers of up to 30B parameters in text-based environments; extending it to larger models and tool-calling or software agents remains future work.
Our harness is also designed by hand, and harnesses that are learned or optimized automatically may interact with HAD differently.
More broadly, HAD reflects a principle that distillation into a system with a fixed component should focus on what that component cannot provide, which may also apply to agents that rely on retrieval or external tools.

\newpage

\bibliographystyle{plainnat}
\bibliography{iclr2027_conference}

@inproceedings{yao2023react,
  title={{ReAct}: Synergizing Reasoning and Acting in Language Models},
  author={Shunyu Yao and Jeffrey Zhao and Dian Yu and Nan Du and Izhak Shafran and Karthik R Narasimhan and Yuan Cao},
  booktitle={International Conference on Learning Representations (ICLR)},
  year={2023}
}

@inproceedings{yang2024sweagent,
  title = {{SWE-agent}: Agent-Computer Interfaces Enable Automated Software Engineering},
  author = {Yang, John and Jimenez, Carlos and Wettig, Alexander and Lieret, Kilian and Yao, Shunyu and Narasimhan, Karthik and Press, Ofir},
  booktitle = {Conference on Neural Information Processing Systems (NeurIPS)},
  year = {2024}
}

@inproceedings{lee2026metaharness,
  title={{Meta-Harness}: End-to-End Optimization of Model Harnesses},
  author={Yoonho Lee and Roshen Sanjay Nair and Qizheng Zhang and Kangwook Lee and Omar Khattab and Chelsea Finn},
  booktitle={Conference on Language Modeling (COLM)},
  year={2026}
}

@inproceedings{wang2025openhands,
  title={{OpenHands}: An Open Platform for {AI} Software Developers as Generalist Agents},
  author={Xingyao Wang and Boxuan Li and Yufan Song and Frank F. Xu and Xiangru Tang and Mingchen Zhuge and Jiayi Pan and Yueqi Song and Bowen Li and Jaskirat Singh and Hoang H. Tran and Fuqiang Li and Ren Ma and Mingzhang Zheng and Bill Qian and Daniel Shao and Niklas Muennighoff and Yizhe Zhang and Binyuan Hui and Junyang Lin and Robert Brennan and Hao Peng and Heng Ji and Graham Neubig},
  booktitle={International Conference on Learning Representations (ICLR)},
  year={2025}
}

@inproceedings{agarwal2024gkd,
  title={On-Policy Distillation of Language Models: Learning from Self-Generated Mistakes},
  author={Rishabh Agarwal and Nino Vieillard and Yongchao Zhou and Piotr Stanczyk and Sabela Ramos Garea and Matthieu Geist and Olivier Bachem},
  booktitle={International Conference on Learning Representations (ICLR)},
  year={2024}
}

@inproceedings{tcod2026,
  title={Exploring Temporal Curriculum in On-Policy Distillation for Multi-turn Autonomous Agents},
  author={Jiaqi Wang and Wenhao Zhang and Weijie Shi and Yaliang Li and James Cheng},
  booktitle={Conference on Language Modeling (COLM)},
  year={2026}
}

@inproceedings{rafailov2023dpo,
  title={Direct preference optimization: Your language model is secretly a reward model},
  author={Rafailov, Rafael and Sharma, Archit and Mitchell, Eric and Manning, Christopher D and Ermon, Stefano and Finn, Chelsea},
  booktitle={Conference on Neural Information Processing Systems (NeurIPS)},
  year={2023}
}

@inproceedings{shridhar2021alfworld,
  title={{ALFWorld}: Aligning Text and Embodied Environments for Interactive Learning},
  author={Mohit Shridhar and Xingdi Yuan and Marc-Alexandre Cote and Yonatan Bisk and Adam Trischler and Matthew Hausknecht},
  booktitle={International Conference on Learning Representations (ICLR)},
  year={2021}
}

@inproceedings{yao2022webshop,
  title={{WebShop}: Towards Scalable Real-World Web Interaction with Grounded Language Agents},
  author={Yao, Shunyu and Chen, Howard and Yang, John and Narasimhan, Karthik},
  booktitle={Conference on Neural Information Processing Systems (NeurIPS)},
  year={2022}
}

@inproceedings{shinn2023reflexion,
  title={Reflexion: language agents with verbal reinforcement learning},
  author={Shinn, Noah and Cassano, Federico and Gopinath, Ashwin and Narasimhan, Karthik and Yao, Shunyu},
  booktitle={Conference on Neural Information Processing Systems (NeurIPS)},
  year={2023}
}

@inproceedings{kang2025distilling,
  title={Distilling {LLM} Agent into Small Models with Retrieval and Code Tools},
  author={Kang, Minki and Jeong, Jongwon and Lee, Seanie and Cho, Jaewoong and Hwang, Sung Ju},
  booktitle={Conference on Neural Information Processing Systems (NeurIPS)},
  year={2025}
}

@inproceedings{yang2024rlcd,
  title={{RLCD}: Reinforcement Learning from Contrastive Distillation for {LM} Alignment},
  author={Kevin Yang and Dan Klein and Asli Celikyilmaz and Nanyun Peng and Yuandong Tian},
  booktitle={International Conference on Learning Representations (ICLR)},
  year={2024}
}

@inproceedings{gu2024minillm,
  title={Mini{LLM}: Knowledge Distillation of Large Language Models},
  author={Yuxian Gu and Li Dong and Furu Wei and Minlie Huang},
  booktitle={International Conference on Learning Representations (ICLR)},
  year={2024}
}

@inproceedings{xu2024contrastive,
  title={Contrastive Preference Optimization: Pushing the Boundaries of {LLM} Performance in Machine Translation},
  author={Xu, Haoran and Sharaf, Amr and Chen, Yunmo and Tan, Weiting and Shen, Lingfeng and Van Durme, Benjamin and Murray, Kenton and Kim, Young Jin},
  booktitle={International Conference on Machine Learning (ICML)},
  year={2024},
}

@inproceedings{zeng2024agenttuning,
  title={{AgentTuning}: Enabling Generalized Agent Abilities for {LLM}s},
  author={Aohan Zeng and Mingdao Liu and Rui Lu and Bowen Wang and Xiao Liu and Yuxiao Dong and Jie Tang},
  booktitle={Findings of the Association for Computational Linguistics: ACL},
  year={2024}
}

@inproceedings{wang2022scienceworld,
  title={{ScienceWorld}: Is your Agent Smarter than a 5th Grader?},
  author={Ruoyao Wang and Peter Jansen and Marc-Alexandre C{\^o}t{\'e} and Prithviraj Ammanabrolu},
  booktitle={Conference on Empirical Methods in Natural Language Processing (EMNLP)},
  year={2022}
}

@inproceedings{liu2026structured,
  title={Structured Agent Distillation for Large Language Model Agents},
  author={Jun Liu and Zhenglun Kong and Peiyan Dong and Changdi Yang and Tianqin Li and Yanyue Xie and Yifan Gong and Xuan Shen and Pu Zhao and Hao Tang and Geng Yuan and Wei Niu and Wenbin Zhang and Xue Lin and Dong Huang and Yanzhi Wang},
  booktitle={International Conference on Autonomous Agents and Multiagent Systems (AAMAS)},
  year={2026}
}

@inproceedings{yan2025rpo,
  title={{RPO}: Retrieval Preference Optimization for Robust Retrieval-Augmented Generation},
  author={Shi-Qi Yan and Quan Liu and Zhen-Hua Ling},
  booktitle={Annual Meeting of the Association for Computational Linguistics (ACL)},
  year={2025}
}

@inproceedings{kim2016sequence,
  title={Sequence-Level Knowledge Distillation},
  author={Yoon Kim and Alexander M. Rush},
  booktitle={Conference on Empirical Methods in Natural Language Processing (EMNLP)},
  year={2016}
}

@misc{belcak2025small,
  title={Small Language Models are the Future of Agentic AI},
  author={Belcak, Peter and Heinrich, Greg and Diao, Shizhe and Fu, Yonggan and Dong, Xin and Muralidharan, Saurav and Lin, Yingyan Celine and Molchanov, Pavlo},
  url={https://arxiv.org/abs/2506.02153v3},
  year={2025}
}

@misc{sopd2026,
  title={Step-Level On-Policy Distillation: Interpolating Between On-Policy Distillation and Supervised Fine-Tuning},
  author={Changhui Sun and Lanbo Liu and Hang Lei and Tong Ling and Jiahang Xie and Zhiyong Zheng and Yujia Wang and Hao Liu and Feng Xiao and Lu Liu and Yanlong Du and Zifeng Cheng and Ziwei Jiang and Qing Gu},
  url={https://arxiv.org/abs/2608.16333v1},
  year={2026}
}

@misc{chen2023fireact,
  title={{FireAct}: Toward language agent fine-tuning},
  author={Chen, Baian and Shu, Chang and Shareghi, Ehsan and Collier, Nigel and Narasimhan, Karthik and Yao, Shunyu},
  url={https://arxiv.org/abs/2310.05915v1},
  year={2023}
}

@misc{sageopd2026,
  title={{SAGE-OPD}: Selective Agent-Guided Intervention for Multi-Turn On-Policy Distillation},
  author={Zhou, Yuhang and Zhang, Lizhu and Wu, Yifan and Wang, Mingyi and Peng, Bo and Liu, Jiayi and Fan, Xiangjun and Zhao, Zhuokai},
  url={https://arxiv.org/abs/2606.19659v1},
  year={2026}
}

@misc{li2026guidedopd,
  title={On-Policy Distillation with Curriculum Turn-level Guidance for Multi-turn Agents},
  author={Gengsheng Li and Mao Zheng and Mingyang Song and Ruiqi Liu and Tianyu Yang and Jie Sun and Qiyong Zhong and Haiyun Guo and Junfeng Fang and Dan Zhang and Jinqiao Wang},
  url={https://arxiv.org/abs/2606.15912v2},
  year={2026}
}

@misc{ye2026onpolicycontext,
  title={On-policy context distillation for language models},
  author={Ye, Tianzhu and Dong, Li and Wu, Xun and Huang, Shaohan and Wei, Furu},
  url={https://arxiv.org/abs/2602.12275v2},
  year={2026}
}

@misc{zhao2026ophsd,
  title={Training with Harnesses: On-Policy Harness Self-Distillation for Complex Reasoning},
  author={Zhao, Zhengyang and Ma, Lu and Zhang, Wentao},
  url={https://arxiv.org/abs/2605.08741v1},
  year={2026}
}

@misc{yu2026pbsd,
  title={Preference-Based Self-Distillation: Beyond {KL} Matching via Reward Regularization},
  author={Yu, Xin and Liao, Liuchen and Zhang, Yiwen and Yu, Yingchen and Xue, Lingzhou and Guo, Qinzhen},
  url={https://arxiv.org/abs/2605.05040v2},
  year={2026}
}

@misc{tu2026ucob,
  title={{UCOB}: Learning to Utilize and Evolve Agentic Skills via Credit-Aware On-Policy Bidirectional Self-Distillation},
  author={Tu, Songjun and Xu, Chengdong and Zhang, Qichao and Ma, Yiwen and Zhang, Yaocheng and Li, Linjing and Li, Dong and Lan, Xiangyuan and Zhao, Dongbin},
  url={https://arxiv.org/abs/2606.29502v2},
  year={2026}
}

@misc{kim2026interplay,
  title={The Interplay of Harness Design and Post-Training in {LLM} Agents},
  author={Kim, Kyungmin and Choi, Youngbin and Lee, Seoyeon and Jun, Suhyeon and Kim, Dongwoo and Park, Sangdon},
  url={https://arxiv.org/abs/2606.25447v1},
  year={2026}
}

@misc{li2024dpkd,
  title={Direct preference knowledge distillation for large language models},
  author={Li, Yixing and Gu, Yuxian and Dong, Li and Wang, Dequan and Cheng, Yu and Wei, Furu},
  url={https://arxiv.org/abs/2406.19774v2},
  year={2024}
}

@misc{yang2025qwen3,
  title={{Qwen3} technical report},
  author={An Yang and Anfeng Li and Baosong Yang and Beichen Zhang and Binyuan Hui and Bo Zheng and Bowen Yu and Chang Gao and Chengen Huang and Chenxu Lv and Chujie Zheng and Dayiheng Liu and Fan Zhou and Fei Huang and Feng Hu and Hao Ge and Haoran Wei and Huan Lin and Jialong Tang and Jian Yang and Jianhong Tu and Jianwei Zhang and Jianxin Yang and Jiaxi Yang and Jing Zhou and Jingren Zhou and Junyang Lin and Kai Dang and Keqin Bao and Kexin Yang and Le Yu and Lianghao Deng and Mei Li and Mingfeng Xue and Mingze Li and Pei Zhang and Peng Wang and Qin Zhu and Rui Men and Ruize Gao and Shixuan Liu and Shuang Luo and Tianhao Li and Tianyi Tang and Wenbiao Yin and Xingzhang Ren and Xinyu Wang and Xinyu Zhang and Xuancheng Ren and Yang Fan and Yang Su and Yichang Zhang and Yinger Zhang and Yu Wan and Yuqiong Liu and Zekun Wang and Zeyu Cui and Zhenru Zhang and Zhipeng Zhou and Zihan Qiu},
  url={https://arxiv.org/abs/2505.09388v1},
  year={2025}
}

@misc{gemmateam2026gemma4,
  title={{Gemma 4} technical report},
  author={{Gemma Team} and Sherif El Abd and Vaibhav Aggarwal and Robin Algayres and Alek Andreev and Olivier Bachem and Ian Ballantyne and Cormac Brick and Victor Cărbune and Michelle Casbon and Mayank Chaturvedi and Aditya Chawla and Victor Cotruta and Alice Coucke and Phil Culliton and Robert Dadashi and Lucas Dixon and Mohamed Elhawaty and Utku Evci and Clément Farabet and Johan Ferret and Filippo Galgani and Sertan Girgin and Jean-Bastien Grill and Maarten Grootendorst and Jiaxian Guo and Cassidy Hardin and Yanzhang He and Steven M. Hernandez and Omri Homburger and Léonard Hussenot and Juyeong Ji and Armand Joulin and Aishwarya Kamath and Parnian Kassraie and Olivier Lacombe and Preethi Lahoti and Gaël Liu and Gus Martins and Luciano Martins and Tatiana Matejovicova and Ramona Merhej and Nikola Momchev and Sneha Mondal and Ryan Mullins and Sindhu Raghuram Panyam and Shreya Pathak and Sarah Perrin and André Susano Pinto and Etienne Pot and Angéline Pouget and Alexandre Ramé and Sabela Ramos and Douglas Reid and David Rim and Morgane Rivière and Karsten Roth and Louis Rouillard and Omar Sanseviero and Pier Giuseppe Sessa and Shane Settle and Danila Sinopalnikov and Sara Smoot and Piotr Stanczyk and Andreas Steiner and Lawrence Stewart and Ilya Tolstikhin and Michael Tschannen and Anton Tsitsulin and Nino Vieillard and Renjie Wu and Pingmei Xu and Haichuan Yang and Edouard Yvinec and Biao Zhang and Li Zhang and Joe Zou and Nicolas Aagnes and Abdelrahman Abdelhamed and Jakub Adamek and Shivani Agrawal and Shubham Agrawal and Ibrahim Alabdulmohsin and Jean Baptiste Alayrac and Uri Alon and Chandramouli Amarnath and Ankesh Anand and Chrysovalantis Anastasiou and Setareh Ariafar and François-Xavier Aubet and Kyriakos Axiotis and Federico Barbero and Joelle Barral and Alexei Bendebury and Urs Bergmann and Stanley Bileschi and Kat Black and Mathieu Blondel and Sebastian Borgeaud and Arthur Bražinskas and Ryan Burnell and Robert Busa-Fekete and Mu Cai and Daniele Calandriello and Glenn Cameron and Charlotte Caucheteux and Rahma Chaabouni and Garima Chadha and Jetha Chan and Blake Jianhang Chen and Jesse Chen and Lin Chen and Xu Chen and Derek Cheng and Tzu-hsiang Chien and Nikolai Chinaev and Yi Chou and Zhaohui Chu and Benjamin Coleman and Pooja Consul and Sam Conway-Rahman and Scott Crowell and Dylan Cutler and Vivek Dani and Samira Daruki and Anil Das and Daniel Deutsch and Nishanth Dikkala and Li Ding and Qiuhan Ding and Shenil Dodhia and Konstantin Donhauser and Tulsee Doshi and Anca Dragan and Alex Druinsky and Sahil Dua and Zoltan Egyed and Danielle Eisenbud and Daniel Eppens and Cindy Fan and Bahare Fatemi and Yassir Fathullah and Vlad Feinberg and Milen Ferev and Sebastian Flennerhag and Takumi Fujimoto and João Gabriel Oliveira and Isaac Galatzer-Levy and João Gante and Simon Geisler and Soham Ghosal and Antonious M. Girgis and Tamara von Glehn and Alec Go and Alhaad Gokhale and Alex Grills and Yiming Gu and Mayank Gupta and Pramod Gupta and Guru Guruganesh and Raia Hadsell and Hamza Harkous and Jitendra Harlalka and Demis Hassabis and Anja Hauth and Joe Heyward and Arian Hosseini and Chih-Yang Hsia and I-Hung Hsu and Xiaopeng Huang and Yangsibo Huang and Kevin Hui and Adrian Hutter and Te I and Fotis Iliopoulos and Advait Jain and Ganesh Jawahar and Ziwei Ji and Qilin Jin and Melvin Johnson and Kandarp Joshi and Arun Kandoor and Wang-Cheng Kang and Koray Kavukcuoglu and Mehran Kazemi and Kathleen Kenealy and Amr Khalifa and Phoebe Kirk and Ivan Korotkov and Suraj Kothawade and Vitaly Kovalev and Neel Kovelamudi and Adam Kraft and Ravin Kumar and Vivek Kumar and Harish Kuppam and Justin Lannin and Chen-Yu Lee and Seungji Lee and Dmitry Lepikhin and Alon Levkovitch and Dongdong Li and Qiujia Li and Valentin Liévin and Ethan Lin and Ziqian Lin and Casper Liu and Tianlin Liu and Tianqi Liu and Xin Liu and Ivan Lobov and Mayank Lunayach and Min Ma and Gagan Madan and Andrii Maksai and Eric Malmi and Michal Matuszak and Daniel McDuff and Gaurav Menghani and Maciej Mikuła and Daniil Mirylenka and Karolis Misiunas and Vedant Misra and Andreea Mitran and Kareem Mohamed and Maksim Mukha and Eric Noland and James O'Donnell and Brendan O'Donoghue and Kate Olszewska and Bernett Orlando and Wanqiong Pan and Rina Panigrahy and Unnati Parekh and Nicolas Perez-Nieves and Chunjong Park and Eric Paskie and Liqian Peng and Bryce Petrini and Slav Petrov and Jonas Pfeiffer and Bilal Piot and Martyna Plomecka and Siim Poder and Octavio Ponce and Arijit Pramanik and David Racz and Anish Rajan and Michelle Ramanovich and Anand Rao and Marvin Ritter and Vitor Rodrigues and Evan Rosen and Mikołaj Rybiński and Noveen Sachdeva and Michaël E. Sander and Rohit Sathyanarayana and Sagar Savla and Samuel Schmidgall and Tal Schuster and George Scrivener and Benoit Seguin and Andrew Sellergren and Aliaksei Severyn and Izhak Shafran and Dhruv Shah and Bobak Shahriari and Yuan Shangguan and Ashish Shenoy and Pradeep Shenoy and Rakesh Shivanna and Pauline Sho and Lucas Spangher and Wojciech Stokowiec and Tim Strother and Yao Su and Yinghao Sun and Mukund Sundararajan and Andrea Tacchetti and Mor Hazan Taege and Pouya Tafti and Jean Tarbouriech and Chetan Tekur and Shantanu Thakoor and Rahul Thapa and Madeleine Traverse and Lenart Treven and Tao Tu and Chien Te Tung and Çağlar Ünlü and Petar Veličković and Malini Pooni Venkat and Sagar Gubbi Venkatesh and Vidya Venkiteswaran and Francesco Visin and Alex Vitvitskyi and Kiran Vodrahalli and Weiyi Wang and Xin Wang and Tris Warkentin and Jan Wassenberg and John Wieting and Cindy Wu and Lechao Xiao and Hao Xu and Yuhui Xu and Fuzhao Xue and Arun Yadav and Jun Yan and Antoine Yang and Lin Yang and Ming-Hsuan Yang and Ziyu Ying and Jae Hyeon Yoo and Morteza Zadimoghaddam and Sajjad Zafar and Fred Zhang and Jiageng Zhang and Jianyi Zhang and Xiaofan Zhang and Chao Zhao and David Zhou and Chen Zou},
  url={https://arxiv.org/abs/2607.02770v2},
  year={2026}
}

@misc{hinton2015distilling,
  title={Distilling the knowledge in a neural network},
  author={Hinton, Geoffrey and Vinyals, Oriol and Dean, Jeff},
  url={https://arxiv.org/abs/1503.02531v1},
  year={2015}
}

@article{lu2025onpolicy,
  author = {Kevin Lu and {Thinking Machines Lab}},
  title = {On-Policy Distillation},
  journal = {Thinking Machines Lab: Connectionism},
  year = {2025},
  note = {https://thinkingmachines.ai/blog/on-policy-distillation},
  doi = {10.64434/tml.20251026},
}

\newpage

\appendix
\crefalias{section}{appendix}
\crefalias{subsection}{subappendix}
\crefalias{subsubsection}{subsubappendix}

\section{Additional Supervision from Harness Awareness}
\label{app:had-theory}

\subsection{Practical Implementation}
\label{app:had-implementation}

\paragraph{Scaling preference margins ($\beta$).}
The margin $\Delta_t$ compares mean log-likelihoods over action tokens, so its scale does not grow with the length of the compared actions.
The coefficient $\beta$ then sets how quickly the logistic loss saturates in this margin.
With a small $\beta$, the loss is close to linear in $\Delta_t$ and every retained pair contributes a similar gradient, whereas a large $\beta$ concentrates the gradient on pairs whose margin is still small or negative.
Since $\lambda_{\mathcal B}$ rescales the preference gradient to a fixed fraction of the distillation gradient, $\beta$ mainly changes how this gradient is distributed across pairs rather than its overall size.
We use $\beta=0.5$ in all environments, selected together with $\rho$ on a validation split held out from the training tasks.
When the two teacher actions are identical, $\Delta_t=0$ for every $\theta$, so the pair contributes a constant loss and no gradient.

\paragraph{Filtering invalid actions ($m_t$).}
The checker $V$ judges the positive teacher action $\tilde a_t^\phi$ only against the records that the harness computes at the visited state $\tilde x_t^\theta$, and uses no task rewards or future information.
At a high level, it rejects an action that the harness records contradict, that is, an action the harness itself would report as impossible or ineffective at this state.
Examples include an action outside the available actions that the harness lists, an action that uses an object the agent does not hold, and an action that repeats one reported as having had no effect at the same state.
The checks for each environment are listed in \cref{app:harness} and \cref{app:implementation_details}.
Negative actions are not checked, since lowering the likelihood of an invalid action does not harm the student.
A rejected pair is removed from the preference term only, and the teacher response at that state still enters $\ell_{\mathrm{dist},t}$.
The checker is therefore conservative, as it removes actions that contradict the records but cannot detect actions that are valid yet do not advance the task (Section~\ref{sec:had-theory}).

\paragraph{Balancing distillation and preference losses ($\lambda$).}
In practice, we compute the mini-batch losses
\begin{equation}
    D_{\mathcal B}=\frac1n\sum_{t\in\mathcal B}\ell_{\mathrm{dist},t},
    \qquad P_{\mathcal B}=\frac{1}{n_V}\sum_{t\in\mathcal B}m_t\ell_{\mathrm{pref},t},
    \qquad n_V=\sum_{t\in\mathcal B}m_t,
\end{equation}
for a minibatch $\mathcal B$ of $n$ records, where $P_{\mathcal B}=0$ when $n_V=0$, and perform backpropagation on the combined loss $D_{\mathcal B} + \lambda_{\mathcal B} P_{\mathcal B}$.
Rather than using a fixed value of $\lambda_{\mathcal B}$, we introduce a per-mini-batch heuristic that scales the newly introduced preference loss-gradient so that its contribution to the parameter update is approximately half that of the original distillation loss-gradient:
\begin{equation}
    \lambda_{\mathcal B}
    =\begin{cases}
        \operatorname{stopgrad}\!\left[
            \rho\dfrac{\lVert\nabla_\theta D_{\mathcal B}\rVert_2}
            {\lVert\nabla_\theta P_{\mathcal B}\rVert_2}\right],
            &\nabla_\theta P_{\mathcal B}\neq0,\\
        0,&\nabla_\theta P_{\mathcal B}=0,
    \end{cases}
    \qquad \rho=0.5.
\end{equation}

\subsection{Ablation Variants}
\label{app:had-ablation-variants}

\paragraph{Response distillation alone and no filtering.}
Setting $\lambda=0$ removes the auxiliary term, whereas setting $m_t\equiv1$ retains every positive preference candidate.

\paragraph{Positive-only action supervision.}
Replace $\ell_{\mathrm{pref},t}$ by
\begin{equation}
    \ell_{\mathrm{pos},t}
    =-\frac{1}{|\tilde a_t^\phi|}\log\pi_\theta(\tilde a_t^\phi\mid\tilde x_t^\theta,\tilde z_t^\theta),
    \label{eq:had-positive-only-loss}
\end{equation}
retaining the positive teacher actions, student reasoning prefixes, length normalization, and validity masks.
This preserves direct student-prefix supervision without the stripped-history comparison.

\paragraph{Full-response preferences.}
Replace $\Delta_{t}$ and $\ell_{\mathrm{pref},t}$ by
\begin{align}
    \Delta_{\mathrm{full},t}
    &=\frac{1}{|\tilde y_t^\phi|}
    \log\pi_\theta(\tilde y_t^\phi\mid\tilde x_t^\theta)
      -\frac{1}{|y_t^\phi|}\log\pi_\theta(y_t^\phi\mid\tilde x_t^\theta),\\
    \ell_{\mathrm{full},t}
    &=-\log\sigma(\beta\Delta_{\mathrm{full},t}).
    \label{eq:had-full-response-preference}
\end{align}

Each target contains its own teacher reasoning and action, and both are scored under the same augmented history.
For $y=(z,a)$, autoregressive factorization gives
\begin{equation}
    \frac{\log\pi_\theta(z,a\mid\tilde x)}{|z|+|a|}
    =\frac{\log\pi_\theta(z\mid\tilde x)
           +\log\pi_\theta(a\mid\tilde x,z)}{|z|+|a|}.
    \label{eq:had-span-weighting}
\end{equation}
Thus this variant changes both the scored span and the action-conditioning prefix, rather than merely adding reasoning supervision to an otherwise identical comparison.

\subsection{What Harness-Aware Action Preference Learning Adds}
\label{app:had-preference-adds}

We fix a realized harness-induced history $\tilde x^\theta$, its stripped counterpart $x^\theta$, and the student reasoning prefix $z_s=\tilde z^\theta$.
Write the teacher response laws and action marginals as
\begin{align}
    \textstyle
    J_+(z,a)=\pi_\phi(z,a\mid\tilde x^\theta),\quad
    P_+(a)=\sum_z J_+(z,a), \\
    \textstyle
    J_-(z,a)=\pi_\phi(z,a\mid x^\theta),\quad
    P_-(a)=\sum_z J_-(z,a).
    \label{eq:had-action-marginals}
\end{align}
Teacher draws are independent of each other and of current student reasoning conditional on the history.
Let $Q(z,a)=Q_Z(z)Q_A(a\mid z)$ be a student response law at $\tilde x^\theta$, where $z$ includes the boundary before the action.
Set $L(a)=|a|>0$ and $D(z,a)=|(z,a)|>0$, and take student likelihoods to be positive wherever they are scored.
The response-distillation risk is
\begin{equation}
    \mathcal D(Q)
    =-\sum_{z,a}\frac{J_+(z,a)}{D(z,a)}\log Q(z,a).
    \label{eq:had-response-risk}
\end{equation}
It depends on the full positive response law $J_+$, but not on the negative teacher law.
Zero-weight imitation terms are interpreted as zero, including at zero student probabilities.

\begin{definition}[Unfiltered preference risk]
\label{def:had-comparison-risk}
For probability laws $p_+$ and $p_-$ on a common finite, nonempty alphabet $\mathcal S$, independent draws $\xi^+\sim p_+$ and $\xi^-\sim p_-$, and scores $g:\mathcal S\to\mathbb R$, define
\begin{equation}
    \mathcal R_{p_+,p_-}(g)
    =\mathbb E[-\log\sigma(g(\xi^+)-g(\xi^-))].
    \label{eq:had-comparison-risk}
\end{equation}
For action comparisons, define the unweighted, unmasked preference risk and the action-preference score by
\begin{equation}
    \mathcal R(f)=\mathcal R_{P_+,P_-}(f),
    \qquad
    f_Q(a)=\frac{\beta}{L(a)}\log Q_A(a\mid z_s),
    \quad\beta>0.
    \label{eq:had-population-risk}
\end{equation}
\end{definition}

\begin{proposition}[Bayes-optimal action-preference scores]
\label{prop:had-preference}
Let $P_+$ and $P_-$ be strictly positive probability laws on a finite, nonempty action set $\mathcal A$, with independent sampling.
The minimizers of $\mathcal R$ over all $f:\mathcal A\to\mathbb R$ are exactly
\begin{equation}
    f^*(a)=s(a)+\kappa,
    \qquad s(a)=\log\frac{P_+(a)}{P_-(a)},
    \quad\kappa\in\mathbb R.
    \label{eq:had-score-optimum}
\end{equation}
\end{proposition}
\begin{proof}
Write $w_{ab}=P_+(a)P_-(b)$, $g(a,b)=f(a)-f(b)$, and
\begin{equation}
    \mu(a,b)=\frac{w_{ab}+w_{ba}}2,
    \qquad \eta(a,b)=\frac{w_{ab}}{w_{ab}+w_{ba}}.
\end{equation}
Symmetrizing the risk gives
\begin{equation}
    \mathcal R(f)
    =\sum_{a,b}\mu(a,b)\operatorname{CE}(\eta(a,b),\sigma(g(a,b))),
    \label{eq:had-risk-cross-entropy}
\end{equation}
where $\operatorname{CE}(p,v)=-p\log v-(1-p)\log(1-v)$.
The weights $\mu$ sum to one, and full support makes every weight positive and every $\eta(a,b)$ strictly between zero and one.
Each cross-entropy is uniquely minimized at its target probability, with log-odds
\begin{equation}
    \log\frac{\eta(a,b)}{1-\eta(a,b)}=s(a)-s(b).
\end{equation}
Thus $f=s$ minimizes all terms simultaneously, and equality requires $f(a)-f(b)=s(a)-s(b)$ for every pair.
Equivalently, $f-s$ is constant.
\end{proof}
The argument also yields the exact excess-risk identity
\begin{equation}
    \mathcal R(f)-\mathcal R(f^*)
    =\mathbb E_{(a,b)\sim\mu}
    D_{\mathrm{KL}}\!\left(
        \operatorname{Ber}(\eta(a,b))\,\|\,
        \operatorname{Ber}(\sigma(f(a)-f(b)))
    \right).
    \label{eq:had-excess-risk}
\end{equation}
This calibrates scores to teacher-input provenance, not to measured action utility.

\begin{observation}[Contrast is not identified by response imitation]
\label{obs:had-additional-contrast}
The positive response law $J_+$ determines $\mathcal D$, but does not generally determine the optimal action-preference score differences, which also depend on $P_-$.
\end{observation}
For example, fix $J_+(z_0,a_1)=J_+(z_0,a_2)=1/2$ and consider two negative action laws,
\begin{equation}
    P_-^{(1)}=(3/4,1/4),\qquad P_-^{(2)}=(1/4,3/4).
\end{equation}
The distillation risks agree for every $Q$, whereas Proposition~\ref{prop:had-preference} gives
\begin{equation}
    f_1^*(a_1)-f_1^*(a_2)=-\log3,
    \qquad f_2^*(a_1)-f_2^*(a_2)=\log3.
    \label{eq:had-contrast-separation}
\end{equation}
The additional query therefore identifies a contrast absent from the positive response law alone.
This does not preclude useful harness-dependent behavior learned through response imitation.

\subsection{Regularization Bias and Reference-Relative Preferences}
\label{app:had-reference-choice}

\paragraph{Regularization bias in isolation.}
Under the assumptions of Proposition~\ref{prop:had-preference}, take $V\equiv1$ and a common action length $L$.
When $P_+=P_-$, the optimal scores are constant, so the isolated preference optimum over normalized action laws is uniform even if the common teacher law is not.
For example, if $P_+=P_-=(0.9,0.1)$ over two actions, the positive--negative pairs $(a_1,a_2)$ and $(a_2,a_1)$ each occur with probability $0.09$.
These symmetric comparisons yield the preference-only optimum $(0.5,0.5)$, rather than the teacher's $(0.9,0.1)$.
Thus reference-free preferences can flatten action probabilities even without a harness-induced action contrast.

\paragraph{A reference-relative variant.}
To address this bias, one can introduce a fixed reference policy $\pi_{\mathrm{ref}}$ that assigns positive probability to every scored action, define
\begin{align}
    \Delta_{\mathrm{rel},t}(\theta)
    &=\frac{1}{|\tilde a_t^\phi|}
      \log\frac{\pi_\theta(\tilde a_t^\phi\mid\tilde x_t^\theta,\tilde z_t^\theta)}
                    {\pi_{\mathrm{ref}}(\tilde a_t^\phi\mid\tilde x_t^\theta,\tilde z_t^\theta)}
    -\frac{1}{|a_t^\phi|}
      \log\frac{\pi_\theta(a_t^\phi\mid\tilde x_t^\theta,\tilde z_t^\theta)}
                    {\pi_{\mathrm{ref}}(a_t^\phi\mid\tilde x_t^\theta,\tilde z_t^\theta)},
      \label{eq:had-reference-relative-margin}\\
    \ell_{\mathrm{rel},t}(\theta)
    &=-\log\sigma\!\left(\beta\Delta_{\mathrm{rel},t}(\theta)\right),
      \label{eq:had-reference-relative-loss}
\end{align}
and replace $\ell_{\mathrm{pref},t}$ by $\ell_{\mathrm{rel},t}$ in Equation~\ref{eq:had-objective}, leaving all other components unchanged.
Under the assumptions above, with student and reference action laws normalized on $\mathcal A$, $P_+=P_-$ now yields the reference law as the isolated optimum.

However, choosing the reference policy is nontrivial: a teacher reference evaluates continuations conditioned on student-generated reasoning, rather than under the marginal $P_+$, while a frozen student may itself make poor use of the harness.
Thus, neither provides a canonical prior over the decision contexts visited by the student.
Consistent with this concern, an ablation that sets $\pi_{\mathrm{ref}}=\pi_{\phi}$ yields worse empirical performance (cf. \cref{app:ablation-extra}).
For equal-length actions, reference-free scoring is equivalent to the uniform reference $\pi_{\mathrm{ref}}\equiv1/|\mathcal A|$.
Note that HAD also retains direct teacher supervision through response distillation, with the preference gradient balanced against that signal.
The flattening bias therefore remains an auxiliary regularization trade-off; the isolated uniform optimum does not imply a uniform solution for the combined objective.

\subsection{What Harness-Aware Validity Filtering Adds}
\label{app:had-filtering-adds}

\begin{definition}[Filtered preference risk]
\label{def:had-filtered-objective}
For a fixed checker $V:\mathcal A\to\{0,1\}$ and independent teacher actions, define
\begin{equation}
    \mathcal R_V(f)
    =\sum_{a,b}V(a)P_+(a)P_-(b)
      \log\!\left(1+e^{f(b)-f(a)}\right).
    \label{eq:had-masked-risk}
\end{equation}
\end{definition}

\begin{observation}[Positive-action filtering]
\label{obs:had-filtering}
Under the full-support and independence assumptions of Proposition~\ref{prop:had-preference}, fix a checker $V$ and let $\gamma=\mathbb E_{P_+}[V(a)]$.
The filtered risk in Definition~\ref{def:had-filtered-objective} satisfies, for $\gamma>0$,
\begin{equation}
    \mathcal R_V(f)=\gamma\mathcal R_{P_+^V,P_-}(f),
    \qquad P_+^V(a)=\frac{V(a)P_+(a)}{\gamma}.
    \label{eq:had-filtered-risk}
\end{equation}
For $\gamma=0$, the risk is identically zero.
For each accepted--rejected pair, only the comparison favoring the accepted action remains, and its contribution strictly decreases as the accepted action's score margin increases.
\end{observation}
\begin{proof}
Substituting $V(a)P_+(a)=\gamma P_+^V(a)$ into the filtered risk gives the identity; when $\gamma=0$, all positive-candidate weights vanish.
If $V(a)=1$ and $V(b)=0$, the oriented comparison weights are $w^V_{ab}=P_+(a)P_-(b)>0$ and $w^V_{ba}=0$.
The retained comparison contributes
\begin{equation}
    w^V_{ab}\log\!\left(1+e^{f(b)-f(a)}\right),
    \label{eq:had-screened-comparison}
\end{equation}
whose derivative with respect to $f(a)-f(b)$ is strictly negative.
\end{proof}
Without filtering, a rejected action can have a larger $P_+/P_-$ ratio than an accepted action.
Validity filtering addresses this failure of harness access as a preference proxy, without requiring checks on negative candidates during training.

\subsection{Why Action-Only Preference Rather Than Full-Response Preference}
\label{app:had-joint-preferences}
Section~\ref{sec:had-action-design} motivates a conditional action comparison at the student's own reasoning prefix.
Full-response preferences instead score each candidate together with its teacher reasoning, changing both the scored span and the action-conditioning context.
We show that this substitution can change the preferred actions even when teacher action marginals stay fixed.
Appendix~\ref{app:had-prefix-supervision} then formalizes how the conditional comparison complements response distillation.
Throughout, teacher reasoning still generates the candidates and the complete positive response remains a distillation target.

\begin{definition}[Full-response preference risk]
\label{def:had-full-response-risk}
Let $J_+$ and $J_-$ be probability laws on a common finite, nonempty response set $\mathcal Y$, whose action projection is $\mathcal A$, with independent teacher draws.
Definition~\ref{def:had-comparison-risk} gives the unweighted, unmasked full-response risk and score
\begin{equation}
    \mathcal R_{\mathrm{full}}(F)=\mathcal R_{J_+,J_-}(F),
    \qquad F_Q(z,a)=\frac{\beta}{D(z,a)}\log Q(z,a).
    \label{eq:had-full-response-risk}
\end{equation}
\end{definition}
This is the population counterpart of Equation~\ref{eq:had-full-response-preference} with $V\equiv1$.

\begin{corollary}[Full-response contrast and conditional action odds]
\label{cor:had-full-response-contrast}
Under Definition~\ref{def:had-full-response-risk}, assume $J_+$ and $J_-$ are strictly positive on $\mathcal Y$.
All minimizers of $\mathcal R_{\mathrm{full}}$ over unrestricted scores are
\begin{equation}
    F^*(z,a)
    =\log\frac{J_+(z,a)}{J_-(z,a)}+\kappa
    =\underbrace{\log\frac{P_+(a)}{P_-(a)}}_{\text{action contrast}}
     +\underbrace{\log\frac{J_+(z\mid a)}{J_-(z\mid a)}}_{\text{conditional reasoning contrast}}
     +\kappa,
    \quad\kappa\in\mathbb R.
    \label{eq:had-reasoning-decomposition}
\end{equation}
If additionally $\mathcal Y=\mathcal Z\times\mathcal A$ and all actions have length $L$, write $d_z=|z|+L$ and $r_z(a)=\log[J_+(z\mid a)/J_-(z\mid a)]$.
Let $q^*_{\mathrm{act}}$ and $Q^*_{\mathrm{full}}$ be the isolated unfiltered optima over positive laws normalized on $\mathcal A$ and $\mathcal Y$, respectively.
Then, for every $z,a,b$,
\begin{align}
    \frac{\beta}{L}\log\frac{q^*_{\mathrm{act}}(a)}{q^*_{\mathrm{act}}(b)}
    &=s(a)-s(b),\\
    \frac{\beta}{d_z}\log\frac{Q^*_{\mathrm{full}}(a\mid z)}{Q^*_{\mathrm{full}}(b\mid z)}
    &=s(a)-s(b)+r_z(a)-r_z(b).
    \label{eq:had-conditional-action-margins}
\end{align}
\end{corollary}
\begin{proof}
Applying Proposition~\ref{prop:had-preference} to the response alphabet and factoring $J_\pm(z,a)=P_\pm(a)J_\pm(z\mid a)$ gives the joint score target.
For positive student laws normalized on $\mathcal Y$, Equation~\ref{eq:had-density-optimum} gives
\begin{equation}
    Q^*_{\mathrm{full}}(z,a)
    =\exp\!\left[
        \frac{D(z,a)}{\beta}
        \left(\log\frac{J_+(z,a)}{J_-(z,a)}+\kappa_{\mathrm{full}}\right)
    \right],
    \label{eq:had-full-probability-optimum}
\end{equation}
with the unique normalizing constant characterized in Section~\ref{app:had-score-probabilities}.
For equal action lengths, Proposition~\ref{prop:had-preference} gives $(\beta/L)\log q^*_{\mathrm{act}}(a)=s(a)+\kappa_{\mathrm{act}}$, while the joint optimum gives $(\beta/d_z)\log Q^*_{\mathrm{full}}(z,a)=s(a)+r_z(a)+\kappa_{\mathrm{full}}$.
Subtracting scores at a shared prefix cancels the normalizing constants and the reasoning marginal $Q^*_{\mathrm{full},Z}(z)$, proving Equation~\ref{eq:had-conditional-action-margins}.
\end{proof}
These identities concern isolated preference optima, not optima of the combined objective.
Because both length factors are positive, an opposite action ordering requires the additional term $r_z(a)-r_z(b)$ to change the sign of the action contrast.
A reasoning contrast constant across actions at the given prefix cancels instead.

\paragraph{Reasoning-based source discrimination.}
For any action score $f$, its lift $F_f(z,a)=f(a)$ satisfies $\mathcal R_{\mathrm{full}}(F_f)=\mathcal R(f)$ after marginalizing reasoning.
Taking $f=s$ therefore yields
\begin{equation}
    \min_F\mathcal R_{\mathrm{full}}(F)
    \leq\min_f\mathcal R(f).
    \label{eq:had-full-risk-comparison}
\end{equation}
Equality holds exactly when the lifted action optimum is also a full-response optimum, requiring $\log(J_+(z\mid a)/J_-(z\mid a))$ to be a single constant on $\mathcal Y$.
Normalization of each conditional law forces that constant to be zero, so equality is equivalent to $J_+(\cdot\mid a)=J_-(\cdot\mid a)$ for every action.
If $P_+=P_-$, the optimal action risk is $\log2$, but the optimal full-response risk is strictly smaller whenever the conditional reasoning laws differ.
For example, $J_\pm(z,a)=P(a)H_\pm(z)$ with distinct, full-support $H_+$ and $H_-$ yields a full-response target depending only on reasoning.
A lower full-response risk can therefore reflect reasoning-based source discrimination without an action-marginal contrast.
This alone need not change the conditional action ordering: the relevant question is whether the additional reasoning contrast differs across actions.

\paragraph{Marginal versus reasoning-conditioned contrasts.}
Under the same full-support assumptions, Bayes' rule gives
\begin{equation}
    s(a)+r_z(a)
    =\log\frac{J_+(a\mid z)}{J_-(a\mid z)}+c_z,
    \qquad c_z=\log\frac{J_{+,Z}(z)}{J_{-,Z}(z)},
    \label{eq:had-conditional-contrast}
\end{equation}
where $J_{\pm,Z}(z)=\sum_aJ_\pm(z,a)$.
The prefix-dependent constant cancels in action comparisons.
Thus, for equal action lengths, the full-response optimum ranks actions at a shared prefix by a reasoning-conditioned action contrast, whereas the action-only optimum uses the marginal contrast $s$.
The conditional contrast may encode useful reasoning--action compatibility; it is not necessarily a nuisance signal.
Our choice of action-only scoring isolates the harness-induced action contrast after marginalizing teacher reasoning, rather than asserting that marginal contrasts always yield better actions.

\begin{observation}[Negative-reasoning invariance]
\label{obs:had-reasoning-invariance}
At a fixed history and student prefix, the combined objective $\mathcal H_V(Q;\lambda)=\mathcal D(Q)+\lambda\mathcal R_V(f_Q)$ depends on negative teacher responses only through $P_-$.
Holding $J_+$, $P_-$, the checker $V$, and the weight $\lambda\geq0$ fixed therefore leaves it unchanged for every admissible $Q$, including parameterized students and variable action lengths.
This follows directly from the loss definitions and concerns negative teacher reasoning, not the student prefix used to score actions.
\end{observation}

\begin{proposition}[Full-response preferences need not preserve action ordering]
\label{prop:had-action-isolation}
With $J_+$ and $P_-$ fixed, changing only negative reasoning can reverse the full-response probability optimum's action ordering both at a shared prefix and marginally over reasoning, even with constant lengths and $V\equiv1$.
The reversed conditional ordering can oppose the action-preference optimum and incur strictly greater action-preference risk.
\end{proposition}
\begin{proof}
Take two one-token actions $a_1,a_2$, two one-token reasoning prefixes $z_0,z_1$, $\beta=2$, and $V\equiv1$, with teacher laws
\begin{equation}
    J_+=\begin{pmatrix}1/3&1/4\\1/3&1/12\end{pmatrix},
    \qquad
    J_-=\begin{pmatrix}1/6&1/24\\1/6&5/8\end{pmatrix}.
    \label{eq:had-full-reversal-teachers}
\end{equation}
Rows index reasoning prefixes and columns index actions.
Both laws have full support, with $P_+=(2/3,1/3)$ and $P_-=(1/3,2/3)$.
Equation~\ref{eq:had-density-optimum}, applied to action scores, gives $q^*_{\mathrm{act}}=(2/3,1/3)$ and score margin $\log4$.
Full responses have length two, so Equation~\ref{eq:had-full-probability-optimum} gives
\begin{equation}
    Q^*_{\mathrm{full}}
    =\begin{pmatrix}15/76&45/76\\15/76&1/76\end{pmatrix},
    \qquad Q^*_{\mathrm{full}}(a_1\mid z_0)=\frac14.
    \label{eq:had-full-action-reversal}
\end{equation}
This favors $a_2$ both at $z_0$ and marginally, since $\sum_zQ^*_{\mathrm{full}}(z,a_1)=15/38<1/2$.
Its action-score margin at $z_0$ is $2\log((1/4)/(3/4))=-\log9$, rather than $\log4$, giving strictly greater action risk by Equation~\ref{eq:had-excess-risk}.
Now replace $J_-$ by $J_-^{(0)}(z,a)=P_-(a)J_+(z\mid a)$, leaving $J_+$ and $P_-$ unchanged.
The joint teacher ratio becomes $P_+(a)/P_-(a)$, independent of reasoning, and
\begin{equation}
    Q^*_{\mathrm{full},0}(a_1\mid z_0)
    =\sum_zQ^*_{\mathrm{full},0}(z,a_1)=\frac45>\frac12.
    \label{eq:had-fixed-supervision-reversal}
\end{equation}
Thus negative reasoning alone reverses both full-response orderings while the conditional HAD objective stays fixed.
\end{proof}
The example holds the positive response law fixed, alongside the action marginals, lengths, and checker acceptance, isolating sensitivity to negative reasoning.
It compares isolated preference optima; it does not establish the same reversal for the full-response variant after adding response distillation, nor does it order the variants by task utility.
More generally, an action-only positive filter leaves $J_+(z\mid a)$ unchanged for accepted actions, so it does not remove their conditional reasoning contrast.
Action-only preferences avoid directly contrasting teacher reasoning, but do not eliminate its role in producing the candidate actions or in the positive response-distillation term.

\subsection{Action-Only Preferences at Student Reasoning Prefixes}
\label{app:had-prefix-supervision}
We now formalize the coupled construction in Section~\ref{sec:had-action-design}: the student's sampled reasoning supplies the context, and teacher actions supply the preference candidates.
The resulting loss supplements imitation of complete teacher responses with supervision of the student's conditional action choice.
All probabilities below are conditional on the fixed harness-induced history.

\begin{definition}[Student-decision preference risk and excess risk]
\label{def:had-prefix-risk}
Let $\mathcal Z$ be a common finite, nonempty reasoning-prefix alphabet.
At every prefix, use the same finite action alphabet $\mathcal A$, positive action lengths $L(a)$, candidate laws $P_\pm$, and checker $V$.
Let $\omega$ be a fixed distribution of collected student prefixes, and draw $z_s\sim\omega$, $(z^+,a^+)\sim J_+$, and $(z^-,a^-)\sim J_-$ independently.
For student action conditionals that are positive laws normalized on $\mathcal A$, define
\begin{align}
    \mathcal R_{V,\omega}(Q)
    &=\mathbb E\!\left[V(a^+)\log\!\left(1+e^{f_{Q,z_s}(a^-)-f_{Q,z_s}(a^+)}\right)\right]\nonumber\\
    &=\sum_z\omega(z)\mathcal R_V(f_{Q,z}),
    \qquad f_{Q,z}(a)=\frac{\beta}{L(a)}\log Q_A(a\mid z).
    \label{eq:had-student-prefix-risk}
\end{align}
For positive action laws $q$ normalized on $\mathcal A$, write $f_q(a)=\beta\log q(a)/L(a)$ and $r_V^\star=\inf_q\mathcal R_V(f_q)$.
Define the per-prefix and prefix-averaged excess risks by
\begin{equation}
    e_V(Q,z)=\mathcal R_V(f_{Q,z})-r_V^\star,
    \qquad
    \mathcal E_{V,\omega}(Q)
    =\sum_z\omega(z)e_V(Q,z)
    =\mathcal R_{V,\omega}(Q)-r_V^\star.
    \label{eq:had-prefix-excess-risk}
\end{equation}
\end{definition}
Using an infimum allows filtered risks whose best separating margins are attained only in a limit.
These excesses concern the isolated action-preference term, not task utility.
The prefix law is held fixed during the comparison, as are collected records during differentiation.

\begin{observation}[Response imitation and student-decision supervision]
\label{obs:had-conditional-supervision}
Under Definition~\ref{def:had-prefix-risk}, let $Q(z,a)=Q_Z(z)Q_A(a\mid z)$ be strictly positive on $\mathcal Z\times\mathcal A$.
Write $w_+(z,a)=J_+(z,a)/D(z,a)$ and $w_{+,Z}(z)=\sum_a w_+(z,a)$.
The response-distillation risk decomposes as
\begin{equation}
    \mathcal D(Q)
    =-\sum_z w_{+,Z}(z)\log Q_Z(z)
     -\sum_{z,a}w_+(z,a)\log Q_A(a\mid z).
    \label{eq:had-response-path-decomposition}
\end{equation}
It therefore scores action conditionals under positive teacher reasoning.
In contrast, for fixed $\omega$, $\mathcal R_{V,\omega}(Q)$ depends on $Q$ only through $Q_A(\cdot\mid z)$ at prefixes with $\omega(z)>0$ and is invariant to changes in $Q_Z$ with these action conditionals held fixed.
For each retained pair, the preference loss is strictly decreasing as a function of its conditional action-score margin $f_{Q,z_s}(a^+)-f_{Q,z_s}(a^-)$.
\end{observation}
\begin{proof}
Substituting $\log Q(z,a)=\log Q_Z(z)+\log Q_A(a\mid z)$ into Equation~\ref{eq:had-response-risk} gives the decomposition.
Equation~\ref{eq:had-student-prefix-risk} contains only action conditionals at prefixes weighted by $\omega$.
For a retained pair with margin $\delta$, the loss is $\log(1+e^{-\delta})$, whose derivative is $-1/(1+e^{\delta})<0$.
No teacher- or student-reasoning likelihood appears as a separate score in that margin.
\end{proof}
The two scoring choices describe one conditional comparison: $z_s$ is the context of the action law being compared, not another response span receiving a preference label.
This is a property of the collected-data loss, not a claim that preference updates leave the student's reasoning unchanged: neural parameters may be shared, and resampling prefixes would change $\omega$.

\begin{corollary}[Independent prefixes preserve the marginal action target]
\label{cor:had-prefix-independence}
Under Definition~\ref{def:had-prefix-risk}, whenever $\omega(z)>0$, the conditional candidate law is $\Pr(a^+=a,a^-=b\mid z_s=z)=P_+(a)P_-(b)$.
Thus the same per-prefix risk $\mathcal R_V$ is used at every supervised prefix.
If additionally $P_\pm$ have full support, $V\equiv1$, and all actions have length $L$, the isolated optimum over independently variable positive normalized action conditionals is
\begin{equation}
    Q_A^*(a\mid z)
    =\frac{\bigl(P_+(a)/P_-(a)\bigr)^{L/\beta}}
           {\sum_b\bigl(P_+(b)/P_-(b)\bigr)^{L/\beta}},
    \qquad \omega(z)>0.
    \label{eq:had-prefix-invariant-optimum}
\end{equation}
In particular, the isolated target is independent of the scoring prefix.
\end{corollary}
\begin{proof}
Independence of the prefix and the two teacher responses gives the product conditional law, so the masked candidate risk at each prefix is $\mathcal R_V$.
For unfiltered comparisons with full support and equal action lengths, Proposition~\ref{prop:had-preference} and Equation~\ref{eq:had-density-optimum} give the stated optimum at every supervised prefix.
\end{proof}
Teacher reasoning generates the candidates but is marginalized in their action law; student reasoning specifies where the same marginal action contrast is taught, not which contrast is taught.
This neither replaces $P_\pm$ by teacher continuation laws conditioned on $z_s$ nor labels the student's reasoning as correct.
Equation~\ref{eq:had-prefix-invariant-optimum} is a relative-preference target, not generally the teacher action law $P_+$.
With variable but prefix-independent action lengths, Equation~\ref{eq:had-density-optimum} likewise yields the same isolated action law at every supervised prefix.
A fixed action-only checker also leaves the per-prefix risk identical, although a filtered infimum may only be approached in a limit.

Importantly, our auxiliary preference objective favors consistency across student reasoning contexts but does not, by itself, preserve useful reasoning-dependent action variation.
This does not imply that the combined HAD optimum is prefix-invariant, as it also includes response distillation and may constrain action conditionals through shared parameters.

\begin{corollary}[Reusing positive teacher reasoning changes the action target]
\label{cor:had-teacher-prefix-target}
Let $(z^+,a^+)\sim J_+$ and $(z^-,a^-)\sim J_-$ be independent teacher responses.
If the positive reasoning $z^+$ is reused as the shared scoring prefix, then, for every $z$ with $J_{+,Z}(z)>0$, the conditional candidate law is $\Pr(a^+=a,a^-=b\mid z^+=z)=J_+(a\mid z)P_-(b)$.
When $J_+(\cdot\mid z)$ and $P_-$ have full support and $V\equiv1$, the optimal unrestricted scores at that prefix are
\begin{equation}
    g_z^*(a)
    =\log\frac{J_+(a\mid z)}{P_-(a)}+\kappa_z
    =s(a)+\log\frac{J_+(z\mid a)}{J_{+,Z}(z)}+\kappa_z.
    \label{eq:had-teacher-prefix-target}
\end{equation}
\end{corollary}
\begin{proof}
Conditioning the positive draw on its reasoning changes its action law to $J_+(\cdot\mid z)$, while independence leaves the negative action law equal to $P_-$.
Proposition~\ref{prop:had-preference} gives the first score identity, and Bayes' rule gives the second.
\end{proof}
This target adds compatibility with the selected positive reasoning trace to the marginal action contrast.
An independently sampled student prefix avoids this extra conditioning; choosing the student rather than another independent prefix law additionally determines the weighting analyzed next.
Full-response preferences are not simply Equation~\ref{eq:had-student-prefix-risk} under a different prefix law: they also change the scored span and couple each candidate to its own teacher reasoning.

\paragraph{Which prefixes receive supervision?}
We now hold the candidate laws $P_\pm$ and checker fixed, changing only the distribution of independently sampled scoring prefixes.
Definition~\ref{def:had-prefix-risk} supplies the common per-prefix excess risk used to compare these distributions.

\begin{proposition}[Prefix coverage and excess-risk transfer]
\label{prop:had-prefix-coverage}
Let $\tau$ be another fixed prefix distribution used with the same action comparison, and define
\begin{equation}
    C_{\omega\mid\tau}
    =\max_{z:\omega(z)>0}\frac{\omega(z)}{\tau(z)},
    \label{eq:had-prefix-coverage-factor}
\end{equation}
with the ratio interpreted as $+\infty$ when $\tau(z)=0$.
For finite $C_{\omega\mid\tau}$ and every admissible student $Q$,
\begin{equation}
    \mathcal E_{V,\omega}(Q)
    \leq C_{\omega\mid\tau}\,\mathcal E_{V,\tau}(Q).
    \label{eq:had-prefix-transfer-bound}
\end{equation}
If $\mathcal R_V(f_q)$ is nonconstant and positive action conditionals can vary independently, this is the smallest uniform factor; when $C_{\omega\mid\tau}=+\infty$, no finite factor suffices.
\end{proposition}
\begin{proof}
Every $e_V(Q,z)$ is nonnegative by definition of $r_V^\star$.
For finite $C_{\omega\mid\tau}$, multiplying $\omega(z)\leq C_{\omega\mid\tau}\tau(z)$ by these excesses and summing proves the bound.
For sharpness, choose a prefix $z_0$ attaining the largest ratio and an action law with strictly positive excess there, which exists because the risk is nonconstant.
At every other prefix, independently choose action laws whose risks approach $r_V^\star$.
The ratio of averaged excesses then approaches $\omega(z_0)/\tau(z_0)$.
If $\tau(z_0)=0<\omega(z_0)$, the $\tau$-averaged excess can approach zero while the $\omega$-averaged excess remains bounded away from zero, ruling out every finite factor.
\end{proof}
The bound applies to parameterized students too; only its sharpness uses independently variable conditionals.
It addresses unequal coverage even when both prefix laws have full support.
For example, with two prefixes, $\tau=(1-\varepsilon,\varepsilon)$, $\omega=(1/2,1/2)$, $0<\varepsilon<1/2$, and excesses $(0,\delta)$ with $\delta>0$, the two averaged excesses are $\varepsilon\delta$ and $\delta/2$.
Such excesses are realizable in the unfiltered full-support setting, where the isolated minimum is attained.
With full prefix coverage and unrestricted conditionals, alternative prefix weights share the same pointwise target and risk infimum.
Their practical difference concerns residual errors due to finite data, limited model capacity, or imperfect optimization, rather than a better pointwise population target.
Student-prefix sampling avoids this transfer step by using the collected prefix law $\omega$ itself; it does not guarantee small excess risk or improved task return.

\paragraph{Uncovered prefixes in teacher-forced losses.}
The decomposition in Observation~\ref{obs:had-conditional-supervision} also makes the zero-coverage case explicit for the actual response losses.
If both teacher reasoning marginals vanish at $z_s$, varying $Q_A(\cdot\mid z_s)$ at fixed $Q_Z$ and other action conditionals leaves both $\mathcal D$ and $\mathcal R_{\mathrm{full}}$ unchanged: no positive-weight term scores a response there.
Filtering positive actions cannot create this missing coverage.
If $\omega(z_s)>0$ and the per-prefix action risk is nonconstant, such a variation can nevertheless change $\mathcal R_{V,\omega}$ through Equation~\ref{eq:had-student-prefix-risk}.
This concerns directly scored conditionals, not the absence of generalization through shared neural parameters.

\paragraph{What the positive-only control retains.}
The loss in Equation~\ref{eq:had-positive-only-loss} retains student-prefix supervision but removes the negative action contrast.
At the fixed history, write $V(a)=V(a\mid\tilde x^\theta)$ for the binary checker and let $q$ be a normalized action law on the finite teacher action set.
Its isolated risk is $-\sum_aV(a)P_+(a)\log q(a)/L(a)$, whose optimum is
\begin{equation}
    q_{\mathrm{pos}}^*(a)
    =\frac{V(a)P_+(a)/L(a)}{\sum_bV(b)P_+(b)/L(b)}
    \label{eq:had-positive-only-optimum}
\end{equation}
when the denominator is positive; if it vanishes, the risk is identically zero.
The weighted-imitation argument in Section~\ref{app:had-score-probabilities} proves the formula, allowing zero probabilities on zero-weight actions.
Thus this control retains direct prefix supervision without learning the $P_+/P_-$ contrast.
It does not by itself isolate the effect of student rather than teacher reasoning conditioning.

\subsection{From Scores to Probabilities}
\label{app:had-score-probabilities}
This section supplies the probability realizations used in the preceding comparisons.
The score characterizations in Proposition~\ref{prop:had-preference} and Equation~\ref{eq:had-reasoning-decomposition} apply to mean token log-likelihoods, not directly to sequence probabilities.
More generally, suppose a risk on a finite, nonempty alphabet $\mathcal S$ has exactly the optimal scores $g+\kappa$.
For positive lengths $\ell(\xi)$ and scores $\beta\log p(\xi)/\ell(\xi)$, its unique optimum over positive probability laws normalized on $\mathcal S$ is
\begin{equation}
    p^*(\xi)=\exp\!\left[\frac{\ell(\xi)}{\beta}\bigl(g(\xi)+\kappa\bigr)\right],
    \qquad \sum_{\xi\in\mathcal S}p^*(\xi)=1.
    \label{eq:had-density-optimum}
\end{equation}
The normalization sum is continuous, strictly increasing in $\kappa$, and ranges from zero to infinity, so exactly one constant realizes the optimal score differences.
For action preferences, take $\mathcal S=\mathcal A$, $g=s$, and $\ell=L$; for full-response preferences, take $\mathcal S=\mathcal Y$, $g=\log(J_+/J_-)$, and $\ell=D$.
When $\ell$ is constant, $p^*\propto\exp(\ell g/\beta)$, preserving the score ordering; with variable lengths, probability and score orderings can differ.
Without normalization on the specified alphabet, score calibration alone does not identify probability mass outside it.
These are isolated preference-loss optima, not optima of the combined HAD objective.
For comparison, minimizing a weighted imitation loss $-\sum_\xi w(\xi)\log p(\xi)$ gives $p^*(\xi)=w(\xi)/\sum_{\xi'}w(\xi')$ whenever $w\geq0$ has positive total mass, allowing boundary probabilities.
This follows by expressing the loss as that total mass times cross-entropy from the normalized weights.

\section{Experiment settings}
\label{app:experimental_settings}

\subsection{Models and tasks}
\label{app:env}

\Cref{tab:app-env} summarizes the environments and models.
All models are public checkpoints served with thinking disabled, and the Gemma pair uses the same prompts, harness, and training recipe as the Qwen3 pair on ALFWorld.
Every environment asks the model for exactly two lines, \texttt{Thought: <brief reasoning>} followed by \texttt{Action: <one command>}.
A reply without a parsable action consumes a turn.
Training rollouts use temperature 1.0, and evaluations use temperature 0.4.

\begin{table}[ht]
\centering
\caption{Environments, models, and evaluation settings.}
\label{tab:app-env}
\small
\begin{tabular}{lccc}
\toprule
 & ALFWorld & WebShop & ScienceWorld \\
\midrule
Teacher & Qwen3-8B & Qwen3-30B-A3B & Qwen3-30B-A3B \\
Student & Qwen3-1.7B & Qwen3-0.6B & Qwen3-0.6B \\
Teacher (other family) & Gemma-4-12B & -- & -- \\
Student (other family) & Gemma-4-E2B & -- & -- \\
Test set & 140 seen, 134 unseen & 500 goals & 149 tasks (30 types) \\
Training pool & 3{,}553 tasks &  11{,}587 goals &  3{,}592 tasks \\
Step limit & 30 & 15 & 15 \\
Max reply tokens & 2{,}048 & 512 & 128 \\
Context length & 24{,}576 & 24{,}576 & 16{,}384 \\
History in the prompt & full & last 2 steps & full \\
Metrics & SR & score, SR & score, SR \\
\bottomrule
\end{tabular}
\end{table}

\paragraph{ALFWorld.}
Each user turn contains the task, the observation, and the list of admissible commands, and the full history is kept as alternating turns.
Harness lines are inserted only in the latest turn, between the observation and the admissible commands.
The task line is the first human annotation of each trial rather than the templated goal, as in the original release.
This annotation differs from the templated goal in 10 of the 274 evaluation tasks, all methods see the same text, and these tasks are not used for the qualitative examples.

\paragraph{WebShop.}
We use the 1{,}000-product version of the environment.
The score is the mean final reward multiplied by 100, and the success rate is the percentage of episodes with reward 1.0.

\paragraph{ScienceWorld.}
We use version 1.2.3 with the teleport simplification and 149 test variations over 30 task types.
The score is the mean final score, which can be negative, and the success rate is the percentage of episodes with a score of at least 100.

\subsection{Harness}
\label{app:harness}

\paragraph{Design.}
All three environments share the same harness code and templates, and only the part that reads the environment state differs.
The harness contains no hints about the task, such as subgoals or solution steps, and reports only what the environment and the agent's own history already reveal.
It only adds text to the observation and never selects or blocks an action.
Some environments already display part of this information in their own observations, such as the list of admissible commands in ALFWorld and WebShop and the interaction history in the prompt.
Since the agent sees this information regardless of the harness, we keep it in every condition, including evaluation without the harness and with a harness category hidden, and the remaining components are added only when the harness is used.

\paragraph{Lines.}
\Cref{tab:app-harness} lists the lines the harness can add at each step.
Each line has three paraphrases, and one of them is chosen at random for each firing with a seed that depends on the episode and step.
The same templates are used in all environments.

\begin{table}[ht]
\centering
\caption{Harness lines. One of the three paraphrases is shown for each line.}
\label{tab:app-harness}
\footnotesize
\begin{tabularx}{\linewidth}{l l X X}
\toprule
Line & Function & Template (one paraphrase) & When it fires \\
\midrule
Situation & State tracking & \texttt{You are at \{location\}. You are holding \{holdings\}. \{m\} of \{n\} places visited; \{k\} unvisited (\{names\}).} & Every step \\
Visit & State tracking & \texttt{This visit to \{location\}: arrived turn \{a\}, \{c\} commands executed since.} and, for a revisit, \texttt{Previous visit: arrived turn \{a\}, \{c\} commands executed, left turn \{l\}.} & Every step with a known location \\
Budget & State tracking & \texttt{Step \{t\} of \{cap\}.} and the warning \texttt{Only \{left\} steps remain. Commit to finishing the task now.} & Counter at every step, warning when fewer than $\theta$ steps remain \\
Failure & Behavior monitoring & \texttt{'\{action\}' did not work here (this was attempt \{N\}). It will keep failing at this state - choose a different action from the displayed list.} & The same action was refused or had no effect at the same state at least the minimum number of times \\
Stall & Behavior monitoring & \texttt{Nothing new has been observed for \{k\} steps. You are not making progress.} & No new state was reached in the last $k$ steps \\
\bottomrule
\end{tabularx}
\end{table}
\paragraph{Thresholds.}
The thresholds are part of the harness design, which is shared by all methods, including the baselines, for both training and evaluation (\cref{tab:app-thresholds}).
They are fixed once for each environment before any distillation, using a validation split held out from the training tasks, and no evaluation task is used.
The stall window $k$ is the smallest value, and the budget threshold $\theta$ the largest value, for which the corresponding line fires on at most 5\% of the steps within successful validation episodes of an untrained model, so that the harness rarely interrupts an agent that is making progress.
The minimum number of failed attempts is chosen with the same criterion.
WebShop does not announce refused actions, so its failure line fires from the first attempt without effect.
This calibration belongs to the harness and not to HAD, and the training objective of HAD uses no task rewards or success labels.

\begin{table}[ht]
\centering
\caption{Harness thresholds for each environment. They are part of the harness shared by all methods and are fixed before distillation on a validation split held out from the training tasks.}
\label{tab:app-thresholds}
\small
\begin{tabular}{lccc}
\toprule
 & ALFWorld & WebShop & ScienceWorld \\
\midrule
Stall window $k$ & 5 & 4 & 8 \\
Budget threshold $\theta$ & 11 & 5 & 11 \\
Minimum failed attempts & 2 & 1 & 2 \\
\bottomrule
\end{tabular}
\end{table}


\paragraph{Example pages.}
\Cref{fig:app-pages} shows one harness-equipped page from each environment, with long parts shortened.

\begin{figure}[ht]
\centering
\footnotesize
\begin{minipage}{\linewidth}
\textbf{ALFWorld} (history omitted)
\begin{verbatim}
Task: Put two remotes on an ottoman.
Observation: Nothing happens.

You are at sofa 1. You are holding remotecontrol 1. 12 of 14 places
visited; 2 unvisited (garbagecan 1, ottoman 1).
This visit to sofa 1: arrived turn 18, 2 commands executed since.
Here, 'take remotecontrol 2 from sofa 1' does nothing (this was
attempt 2). This state will not accept it - select another action
from the displayed list.
Step 21 of 30.
Only 9 steps remain. Commit to finishing the task now.

Admissible commands:
examine remotecontrol 1 / examine sofa 1 / go to armchair 1 / ...
\end{verbatim}
\textbf{WebShop}
\begin{verbatim}
Observation: [item page: The Spicy Beef Backpacking Bundle ...
Price: $14.49 ... [button] Buy Now]

You are at item B0978Q1KK9 / page. You are holding nothing. 3 of 14
places visited; 11 unvisited (home, b0978q1kk9, ..., +6 more).
Not yet chosen here: flavor name.
This visit to item B0978Q1KK9 / page: arrived turn 6, 1 commands
executed since.
Previous visit: arrived turn 3, 1 commands executed, left turn 4.
Nothing new has been observed for 4 steps. You are not making
progress.
Step 7 of 15.

Admissible commands: click[back to search] / click[< prev] / ...
\end{verbatim}
\textbf{ScienceWorld}
\begin{verbatim}
Action: move baby wolf to yellow box in living room
Observation: No known action matches that input.

Location: outside. In hand: orange. 2 of 4 places visited;
2 unvisited (greenhouse, kitchen).
This visit to outside: arrived turn 0, 5 commands executed since.
Here, 'move baby wolf to yellow box in living room' does nothing
(this was attempt 2). This state will not accept it - select
another action from the displayed list.
Step 5 of 15.
\end{verbatim}
\end{minipage}
\caption{Example pages that the zero-shot teacher sees with the harness. Long parts of the pages are shortened.}
\label{fig:app-pages}
\end{figure}

\subsection{Harness utilization}
\label{app:hu}

Harness utilization counts only the components whose report implies a rule that holds for every task.
The visit component and the visited places in the situation component are not included, since the right response to them depends on the task, and neither is the budget counter, which only states the current step.
Each of the five remaining components is scored separately, and harness utilization is their unweighted mean (\cref{eq:hu}), so that components applying to every step do not outweigh those that fire rarely.

For each component $c \in \mathcal{C}$ with a task-agnostic rule $r_c(s) \in \{0,1\}$, harness utilization averages the fraction of the steps $S_c$ where the component applies at which the rule holds,
\begin{equation}
    \mathrm{HU} = \frac{1}{|\mathcal{C}|}\sum_{c\in\mathcal{C}} \frac{1}{|S_c|}\sum_{s\in S_c} r_c(s).
    \label{eq:hu}
\end{equation}
The rules require the next action to be available, to use only held objects, to avoid purely informational actions after a budget warning, to differ from an action reported as failed, and to reach a new state within $k$ steps after a stall report.

A step $s$ is counted for a component if the page shown before action $s{+}1$ carries that component and its rule can be evaluated.
Seen and unseen ALFWorld tasks are pooled.
An action outside the ALFWorld command grammar is counted as inconsistent for available actions and is not counted for the other components, so that each error is counted once.

\paragraph{Rules.}
Available actions applies to every step, and the action is consistent if it is in the admissible commands shown at that step.
The situation component applies to every step, and the action is inconsistent if it places, heats, cools, or cleans an object the agent is not holding, slices with a tool the agent is not holding, or takes an object while already holding one.
The budget component applies to steps with a budget warning, and the action is inconsistent if it is \texttt{look}, \texttt{inventory}, \texttt{examine}, or \texttt{help}.
Failure applies to steps with a failure report, and the action is inconsistent if it equals the reported action.
Stall applies to steps with a stall report, and the step is consistent if a state not seen earlier in the episode is reached within the next $k=5$ steps or the episode ends with success.
States are identified by the same key the harness uses, namely the location, the held objects, and the admissible commands.

\begin{table}[ht]
\centering
\caption{Harness utilization and the rate of each component on ALFWorld with the harness.}
\label{tab:app-hu}
\small
\begin{tabular}{lcccccc}
\toprule
 & Available actions & Situation & Budget & Failure & Stall & HU \\
\midrule
Teacher & 90.1 & 97.2 & 87.8 & 79.4 & 16.9 & 74.3 \\
Student & 73.7 & 96.2 & 79.1 & 61.2 & 18.1 & 65.7 \\
\midrule
OPD & 87.8 & 97.0 & 88.4 & 52.9 & 17.6 & 68.8 \\
SAGE-OPD & 88.0 & 97.9 & 83.7 & 52.3 & 11.0 & 66.6 \\
Guided-OPD & 82.5 & 94.1 & 85.7 & 28.6 & 5.3 & 59.2 \\
SOPD & 84.8 & 96.5 & 92.3 & 67.3 & 24.7 & 73.1 \\
\textbf{HAD} & \textbf{93.3} & \textbf{99.0} & 87.7 & \textbf{91.1} & \textbf{33.8} & \textbf{81.0} \\
\bottomrule
\end{tabular}
\end{table}

\subsection{Implementation details}
\label{app:implementation_details}

\paragraph{Common implementation.}
All baselines and HAD run in one synchronous on-policy loop.
In each round, the current student rolls out 16 tasks at temperature 1.0 with the harness on the page, and the teacher reads the same harness-equipped page.
Collected steps are kept for two policy versions and each is used once.
All methods use full fine-tuning with the settings in \cref{tab:app-hparams}.
For methods based on OPD, the KL divergence is computed over the teacher's top 64 tokens at each position instead of the full vocabulary.
The original papers use a learning rate of $10^{-6}$, which did not train the students in our setting, so we use $10^{-5}$ for all methods.

\paragraph{Teacher sampling, truncation, and validity check.}
Following prior on-policy distillation work for multi-turn agents~\citep{li2026guidedopd,sopd2026}, student rollouts are sampled at temperature 1.0 and all evaluations use temperature 0.4.
The teacher samples its responses with and without the harness independently at the same temperature 1.0, so that the two teacher actions are draws from the action laws $P_+$ and $P_-$ analyzed in \cref{app:had-theory}.
Since both teacher responses are sampled, the rates in \cref{app:disagreement} at which the two teacher actions differ include differences due to sampling as well as those due to the harness.
For the methods based on OPD, the KL divergence is computed over the teacher's top 64 tokens at each position, with both distributions renormalized over these tokens.
The validity checker of HAD applies the same checks in all three environments, rejecting a positive action that is not among the admissible commands, that uses an object the agent does not hold, or that repeats an action reported as having had no effect at the same state.

\paragraph{Baseline implementations.} For OPD, the student is trained with the token-level reverse KL divergence to the teacher on its own sampled responses, the on-policy case of GKD~\citep{agarwal2024gkd}. For SAGE-OPD, following the original paper~\citep{sageopd2026}, each student turn is weighted by the teacher's judgment of whether it needs no, weak, or strong correction, multiplied by the teacher's confidence over the turn, and the weights are normalized in each update. For Guided-OPD, following the original paper~\citep{li2026guidedopd}, the teacher takes over each turn with a probability that decays from 1 to 0 with a cosine schedule over the first 80\% of updates, and student turns are trained with the OPD loss.
Teacher turns are trained with cross-entropy on the teacher's decoded tokens instead of the forward KL divergence of the original paper. Finally, for SOPD, following the original paper~\citep{sopd2026}, the teacher writes one full step at every state the student visits, and the student is trained with cross-entropy on these steps, averaged over the tokens of each step and then over steps.

\begin{table}[ht]
\centering
\caption{Training hyperparameters, shared by all methods unless noted.}
\label{tab:app-hparams}
\small
\begin{tabular}{lccc}
\toprule
 & ALFWorld & WebShop & ScienceWorld \\
\midrule
Optimizer & \multicolumn{3}{c}{AdamW, weight decay 0.01, gradient clipping 1.0} \\
Learning rate & $10^{-5}$ & $10^{-5}$ & $10^{-5}$ \\
Schedule & \multicolumn{3}{c}{cosine to 0, 15 warmup updates} \\
Effective batch size & 60 & 60 & 60 \\
Updates & 250 & 150 & 150 \\
Tasks per round & 16 & 16 & 16 \\
Staleness & 2 & 2 & 2 \\
Rollout temperature & 1.0 & 1.0 & 1.0 \\
Evaluation temperature & 0.4 & 0.4 & 0.4 \\
\midrule
HAD $\beta$ & 0.5 & 0.5 & 0.5 \\
HAD $\rho$ & 0.5 & 0.5 & 0.5 \\
\bottomrule
\end{tabular}
\end{table}

\section{Prompts and Examples}
\label{app:prompts}

We show the prompts and replies of each method for one success case and one failure case on ALFWorld.
Each prompt is shown as recorded, with the role markers \texttt{<system>}, \texttt{<user>}, and \texttt{<assistant>}, and the server adds the chat template tokens of each model.
The instructions are the same in all prompts, so each figure replaces them with \texttt{[... instructions elided ...]}, and \cref{fig:app-instructions} shows them in full.
Each figure keeps only the last three turns and shows the admissible commands only in the last turn.
The harness lines appear in the last user turn, between the observation and the admissible commands.
The task line is the first human annotation of the trial, while the templated goal appears once in the first observation of each episode.

\begin{figure}[t]
\begin{plist}[promptstyle]{Instructions shared by all prompts}
<system>You are an agent acting in a text-based ALFWorld household environment. Your goal is to complete the household task by choosing one environment action at a time.

On every turn, respond with exactly two lines and nothing else:

Thought: <brief reasoning about the next step>
Action: <one command>

Choose the Action from the admissible commands provided for the current state. Copy the command exactly, including object and receptacle names and their numeric identifiers. Issue exactly one action.

If the environment returns "Nothing happens.", the previous action had no effect.

Use the current observation and the interaction history to decide what to do next. Do not add any text before the Thought line or after the Action line.
\end{plist}
\caption{Instructions elided from the per-case figures. They are identical across all ALFWorld cases.}
\label{fig:app-instructions}
\end{figure}

\begin{figure}[p]
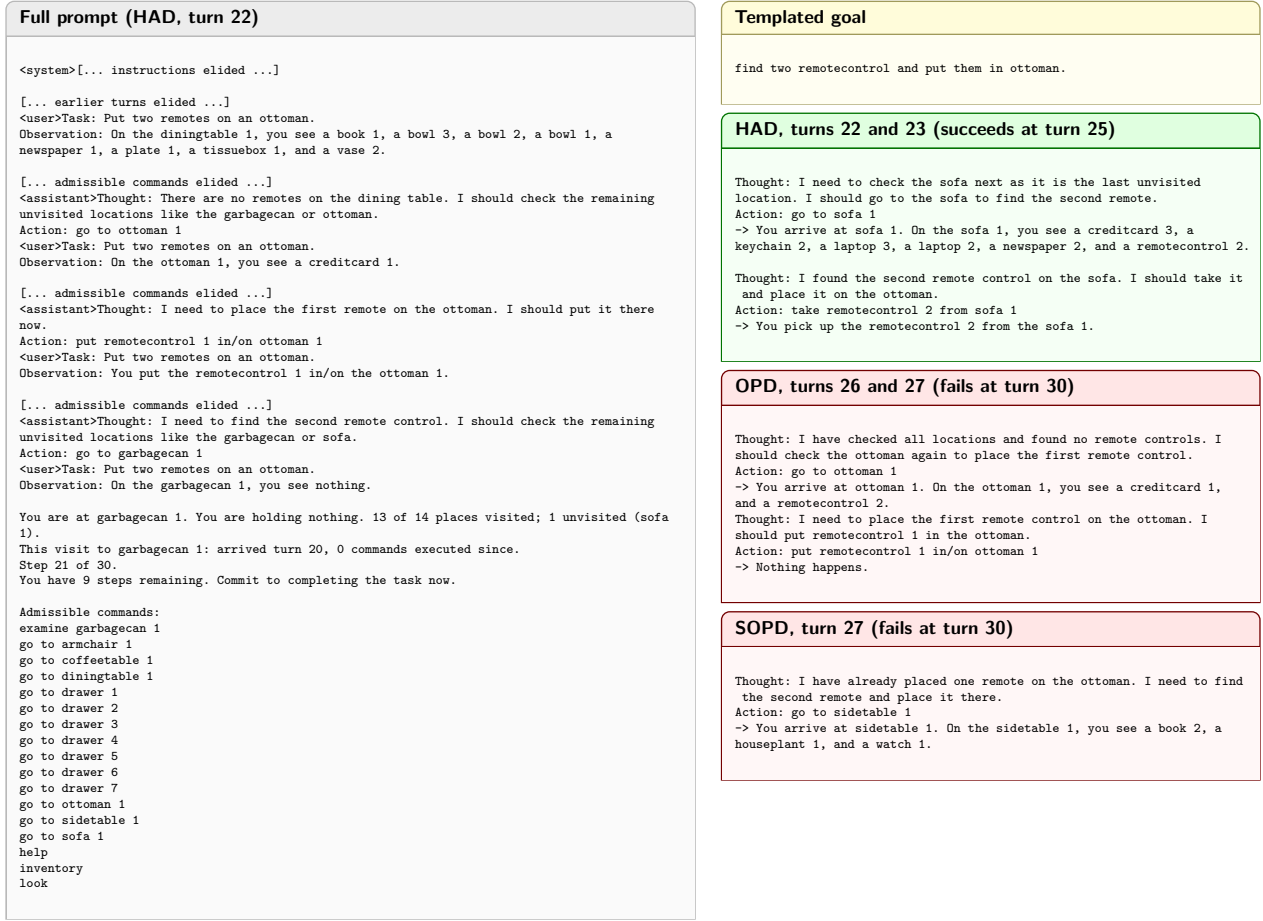

\begin{minipage}[t]{0.55\linewidth}\vspace{0pt}
\begin{plist}[promptstyle]{Full prompt (HAD, turn 22)}
<system>[... instructions elided ...]

[... earlier turns elided ...]
<user>Task: Put two remotes on an ottoman.
Observation: On the diningtable 1, you see a book 1, a bowl 3, a bowl 2, a bowl 1, a newspaper 1, a plate 1, a tissuebox 1, and a vase 2.

[... admissible commands elided ...]
<assistant>Thought: There are no remotes on the dining table. I should check the remaining unvisited locations like the garbagecan or ottoman.
Action: go to ottoman 1
<user>Task: Put two remotes on an ottoman.
Observation: On the ottoman 1, you see a creditcard 1.

[... admissible commands elided ...]
<assistant>Thought: I need to place the first remote on the ottoman. I should put it there now.
Action: put remotecontrol 1 in/on ottoman 1
<user>Task: Put two remotes on an ottoman.
Observation: You put the remotecontrol 1 in/on the ottoman 1.

[... admissible commands elided ...]
<assistant>Thought: I need to find the second remote control. I should check the remaining unvisited locations like the garbagecan or sofa.
Action: go to garbagecan 1
<user>Task: Put two remotes on an ottoman.
Observation: On the garbagecan 1, you see nothing.

You are at garbagecan 1. You are holding nothing. 13 of 14 places visited; 1 unvisited (sofa 1).
This visit to garbagecan 1: arrived turn 20, 0 commands executed since.
Step 21 of 30.
You have 9 steps remaining. Commit to completing the task now.

Admissible commands:
examine garbagecan 1
go to armchair 1
go to coffeetable 1
go to diningtable 1
go to drawer 1
go to drawer 2
go to drawer 3
go to drawer 4
go to drawer 5
go to drawer 6
go to drawer 7
go to ottoman 1
go to sidetable 1
go to sofa 1
help
inventory
look
\end{plist}
\end{minipage}\hfill
\begin{minipage}[t]{0.43\linewidth}\vspace{0pt}
\begin{plist}[envstyle]{Templated goal}
find two remotecontrol and put them in ottoman.
\end{plist}
\begin{plist}[hadstyle]{HAD, turns 22 and 23 (succeeds at turn 25)}
Thought: I need to check the sofa next as it is the last unvisited location. I should go to the sofa to find the second remote.
Action: go to sofa 1
-> You arrive at sofa 1. On the sofa 1, you see a creditcard 3, a keychain 2, a laptop 3, a laptop 2, a newspaper 2, and a remotecontrol 2.
Thought: I found the second remote control on the sofa. I should take it and place it on the ottoman.
Action: take remotecontrol 2 from sofa 1
-> You pick up the remotecontrol 2 from the sofa 1.
\end{plist}
\begin{plist}[basestyle]{OPD, turns 26 and 27 (fails at turn 30)}
Thought: I have checked all locations and found no remote controls. I should check the ottoman again to place the first remote control.
Action: go to ottoman 1
-> You arrive at ottoman 1. On the ottoman 1, you see a creditcard 1, and a remotecontrol 2.
Thought: I need to place the first remote control on the ottoman. I should put remotecontrol 1 in the ottoman.
Action: put remotecontrol 1 in/on ottoman 1
-> Nothing happens.
\end{plist}
\begin{plist}[basestyle]{SOPD, turn 27 (fails at turn 30)}
Thought: I have already placed one remote on the ottoman. I need to find the second remote and place it there.
Action: go to sidetable 1
-> You arrive at sidetable 1. On the sidetable 1, you see a book 2, a houseplant 1, and a watch 1.
\end{plist}
\end{minipage}
\caption{ALFWorld success case, seen task \texttt{trial\_T20190906\_182952\_858723}. HAD follows the harness report that the sofa is the only unvisited place and finds the second remote. The baselines had passed the sofa earlier without taking the remote, and their replies are shown at the corresponding point near the end of their episodes. OPD tries to place a remote it is not holding, and SOPD returns to a place it has already searched.}
\label{fig:app-case1}
\end{figure}

\begin{figure}[p]
\begin{minipage}[t]{0.55\linewidth}\vspace{0pt}
\begin{plist}[promptstyle]{Full prompt (HAD, turn 28)}
<system>[... instructions elided ...]

[... earlier turns elided ...]
<user>Task: Move a heated mug to the base of the coffee maker.
Observation: On the coffeemachine 1, you see a mug 1.

[... admissible commands elided ...]
<assistant>Thought: The task is complete as the mug is now in the coffee machine.
Action: inventory
<user>Task: Move a heated mug to the base of the coffee maker.
Observation: You are not carrying anything.

[... admissible commands elided ...]
<assistant>Thought: I need to confirm that the mug is now in the coffee machine and that the task is complete.
Action: examine coffeemachine 1
<user>Task: Move a heated mug to the base of the coffee maker.
Observation: On the coffeemachine 1, you see a mug 1.

[... admissible commands elided ...]
<assistant>Thought: The task is complete as the mug is now in the coffee machine.
Action: go to countertop 1
<user>Task: Move a heated mug to the base of the coffee maker.
Observation: Nothing happens.

Location: coffeemachine 1. In hand: nothing. 12 of 25 places visited; 13 unvisited (countertop 2, countertop 3, garbagecan 1, microwave 1, shelf 1, +8 more).
This visit to coffeemachine 1: arrived turn 20, 5 commands executed since.
Nothing new has been observed for 5 steps. You are not making progress.
Step 27 of 30.
You have 3 steps remaining. Commit to completing the task now.

Admissible commands:
examine coffeemachine 1
examine countertop 1
go to cabinet 1
go to cabinet 2
go to cabinet 3
go to cabinet 4
go to cabinet 5
go to cabinet 6
go to countertop 2
go to countertop 3
go to drawer 1
go to drawer 2
go to drawer 3
go to fridge 1
go to garbagecan 1
go to microwave 1
go to shelf 1
go to shelf 2
go to shelf 3
go to sinkbasin 1
go to stoveburner 1
go to stoveburner 2
go to stoveburner 3
go to stoveburner 4
go to toaster 1
help
inventory
look
take mug 1 from coffeemachine 1
take pan 1 from countertop 1
take spoon 1 from countertop 1
\end{plist}
\end{minipage}\hfill
\begin{minipage}[t]{0.43\linewidth}\vspace{0pt}
\begin{plist}[envstyle]{Templated goal}
put a hot mug in coffeemachine.
\end{plist}
\begin{plist}[basestyle]{HAD, turn 28 (fails at turn 30)}
Thought: I need to confirm that the mug is now in the coffee machine and that the task is complete.
Action: inventory
-> You are not carrying anything.
\end{plist}
\begin{plist}[hadstyle]{OPD, turn 28 (succeeds at turn 28)}
Thought: I need to place the heated mug in the coffee machine.
Action: put mug 2 in/on coffeemachine 1
-> You put the mug 2 in/on the coffeemachine 1.
\end{plist}
\end{minipage}
\caption{ALFWorld failure case, unseen task \texttt{trial\_T20190907\_171933\_349922}. Earlier in the episode, HAD placed an unheated mug on the coffee machine and concluded that the task was done. The harness reports a lack of progress, but HAD repeats \texttt{inventory}, one of the actions it has been cycling through, and fails at the step limit. OPD succeeds, but it is in a different state at this turn, holding a heated mug, so its reply does not answer the same page.}
\label{fig:app-case2}
\end{figure}
\section{Further experiments}
\label{app:further_experiments}

\subsection{Analysis: turns}
\label{app:turns}

\Cref{tab:app-turns} reports the mean number of turns over successful ALFWorld episodes, evaluated with and without the harness.
With the harness, HAD takes a similar number of turns to the other methods while solving the most tasks.
Adding the harness increases the number of turns for most models, including the teacher and the untrained student.

\begin{table}[ht]
\centering
\caption{Mean number of turns over successful ALFWorld episodes.}
\label{tab:app-turns}
\small
\begin{tabular}{lcc}
\toprule
 & With harness & Without harness \\
\midrule
Teacher & 14.10  & 12.99 \\
Student & 13.86 & 9.77 \\
\midrule
OPD & 13.16 & 13.45  \\
SAGE-OPD & 14.71 & 12.85  \\
Guided-OPD & 14.14  & 12.23 \\
SOPD & 15.15 & 12.31  \\
\textbf{HAD} & 14.83  & 13.14  \\
\bottomrule
\end{tabular}
\end{table}

\subsection{Generalization to other model architecture}
\label{app:gemma}

\Cref{tab:app-gemma} reports the full results with Gemma models, including the number of turns and harness utilization.
With the harness, HAD achieves the highest success rate on both splits and the highest harness utilization at 77.9\%, above even the 75.9\% of the teacher.
Every OPD baseline instead ends below the 71.4\% of the untrained student in harness utilization, so adding the harness to standard distillation does not make the Gemma student use it better.
Without the harness, HAD loses more than the baselines on unseen tasks, where SOPD reaches 29.9\% against 23.9\% for HAD, which is consistent with HAD relying on the harness that stays in place at deployment.
HAD also takes more turns in successful episodes than the other methods, while solving the most tasks.

\begin{table}[ht]
\centering
\caption{Full results on ALFWorld with Gemma models (E2B $\leftarrow$ 12B). SR cells report performance \textbf{with harness} \textcolor{gray!90}{\scriptsize/\,without harness}. Turns (mean over successful episodes) and harness utilization (HU) are measured with the harness.}
\label{tab:app-gemma}
\setlength{\aboverulesep}{0pt}
\setlength{\belowrulesep}{0pt}
\setlength{\extrarowheight}{0.15em}
\setlength{\tabcolsep}{0.6em}
\renewcommand{\arraystretch}{1.1}
\small
\begin{tabular}{l*{2}{>{\centering\arraybackslash}p{2.0cm}}*{2}{>{\centering\arraybackslash}p{1.3cm}}}
\toprule
Method & Unseen SR & Seen SR & Turns & HU \\
\midrule
\textit{Teacher} & \hcell{51.49}{43.28} & \hcell{49.29}{42.14} & \hcell{13.91}{13.93} & 75.92 \\
\textit{Student} & \hcell{14.93}{8.21} & \hcell{18.57}{15.71} & \hcell{12.26}{12.92} & 71.38 \\
\midrule
OPD & \hcell{26.87}{18.66} & \hcell{31.43}{29.29} & \hbcell{13.26}{12.97} & 68.76 \\
SAGE-OPD & \hcell{24.63}{20.90} & \hcell{28.57}{27.14} & \hcell{14.30}{13.79} & 67.35 \\
Guided-OPD & \hcell{29.10}{18.66} & \hcell{32.86}{24.29} & \hcell{14.84}{11.58} & 64.05 \\
SOPD & \hcell{32.09}{29.85} & \hcell{34.29}{25.71} & \hcell{13.85}{12.33} & 69.04 \\
\sk \textbf{HAD (ours)} & \sk \hbcell{37.31}{23.88} & \sk \hbcell{36.43}{30.00} & \sk \hcell{15.42}{12.99} & \sk \bf{77.91} \\
\bottomrule
\end{tabular}
\end{table}

\subsection{Additional ablations}
\label{app:ablation-extra}

We report two further ablations on ALFWorld, with all other settings fixed.
\Cref{tab:app-design} compares HAD with a full-response contrast and with the teacher as a reference model, whose definitions are given in \cref{app:had-ablation-variants,app:had-reference-choice} and whose results are discussed in \cref{sec:ablation}.
\Cref{tab:app-rho} varies the preference weight $\rho$ from 0.25 to 2.0, and the success rate is highest at $\rho=0.5$, the value used in all main experiments.

\begin{table}[ht]
\centering
\small
\begin{minipage}[t]{0.55\linewidth}
\centering
\caption{Design choices of the preference term on ALFWorld, with the harness.}
\label{tab:app-design}
\begin{tabular}{lcc}
\toprule
Variant & Unseen SR & Seen SR \\
\midrule
Full-response contrast & 59.7 & 47.1 \\
Teacher as reference & 56.0 & 46.4 \\
\textbf{HAD (ours)} & \textbf{63.4} & \textbf{51.4} \\
\bottomrule
\end{tabular}
\end{minipage}\hfill
\begin{minipage}[t]{0.40\linewidth}
\centering
\caption{Preference weight $\rho$ on ALFWorld, with the harness.}
\label{tab:app-rho}
\begin{tabular}{lcc}
\toprule
$\rho$ & Unseen SR & Seen SR \\
\midrule
0.25 & 53.7 & 47.9 \\
\textbf{0.5} & \textbf{63.4} & \textbf{51.4} \\
1.0 & 49.3 & 50.7 \\
2.0 & 51.5 & 49.3 \\
\bottomrule
\end{tabular}
\end{minipage}
\end{table}

\subsection{Where the harness changes the teacher's action}
\label{app:disagreement}

\begin{figure}[ht]
  \centering
  \includegraphics[width=0.4\linewidth]{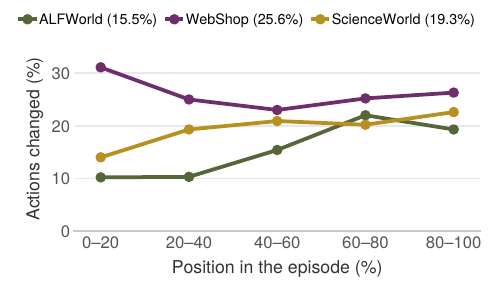}
    \caption{How often the teacher actions with and without the harness differ at each position in the episode.}  
\label{fig:disagreement}
\end{figure}

We measure how often the teacher chooses a different action with and without the harness at the states visited during HAD training.
As \cref{fig:disagreement} shows, the two teacher actions differ at 15.5\% of the states on ALFWorld, 25.6\% on WebShop, and 19.3\% on ScienceWorld.
These states are spread over the whole episode rather than concentrated at its start.

\end{document}